\documentclass{article} 
\pdfoutput=1
\usepackage[utf8]{inputenc}
\usepackage[english]{babel}
\usepackage{main,times}

\usepackage{amsmath,amsfonts,bm}

\def\eqref#1{equation~\ref{#1}}

\def\1{\bm{1}}

\DeclareMathAlphabet{\mathsfit}{\encodingdefault}{\sfdefault}{m}{sl}
\SetMathAlphabet{\mathsfit}{bold}{\encodingdefault}{\sfdefault}{bx}{n}

\newcommand{\E}{\mathbb{E}}

\newcommand{\sigmoid}{\sigma}

\usepackage{hyperref}
\usepackage{url}
\usepackage{multirow}
\usepackage[most]{tcolorbox}
\usepackage{xcolor}
\usepackage{setspace}
\usepackage{subcaption}
\usepackage{amsthm}
\usepackage{enumitem}
\usepackage{tcolorbox}
\usepackage{float}
\usepackage{relsize}
\usepackage{booktabs}
\usepackage{placeins}
\usepackage{CJKutf8}

\newtheorem{proposition}{Proposition}
\newtheorem{corollary}{Corollary}
\numberwithin{proposition}{section}
\numberwithin{corollary}{section}

\title{Diverse Preference Learning for Capabilities and Alignment}

\author{Stewart Slocum, Asher Parker-Sartori, and Dylan Hadfield-Menell \\
MIT CSAIL \\
\texttt{\{stew, asher, dhm\}@csail.mit.edu}
}

\iclrfinalcopy 
\begin{document}

\maketitle

\begin{abstract}
The ability of LLMs to represent diverse perspectives is critical as they increasingly impact society. However, recent studies reveal that alignment algorithms such as RLHF and DPO significantly reduce the diversity of LLM outputs. Not only do aligned LLMs generate text with repetitive structure and word choice, they also approach problems in more uniform ways, and their responses reflect a narrower range of societal perspectives. We attribute this problem to the KL divergence regularizer employed in preference learning algorithms. This causes the model to systematically overweight majority opinions and sacrifice diversity in its outputs. To address this, we propose Soft Preference Learning, which decouples the entropy and cross-entropy terms in the KL penalty —
allowing for fine-grained control over LLM generation diversity. From a capabilities perspective, LLMs trained using Soft Preference Learning attain higher accuracy on difficult repeated sampling tasks and produce outputs with greater semantic and lexical diversity. From an alignment perspective, they are capable of representing a wider range of societal viewpoints and display improved logit calibration. Notably, Soft Preference Learning resembles, but is a Pareto improvement over, standard temperature scaling.
\end{abstract}

\section{Introduction}

\begin{figure}[t]
    \centering
    \begin{subfigure}[t]{0.67\linewidth}
        \centering
        \includegraphics[width=\linewidth]{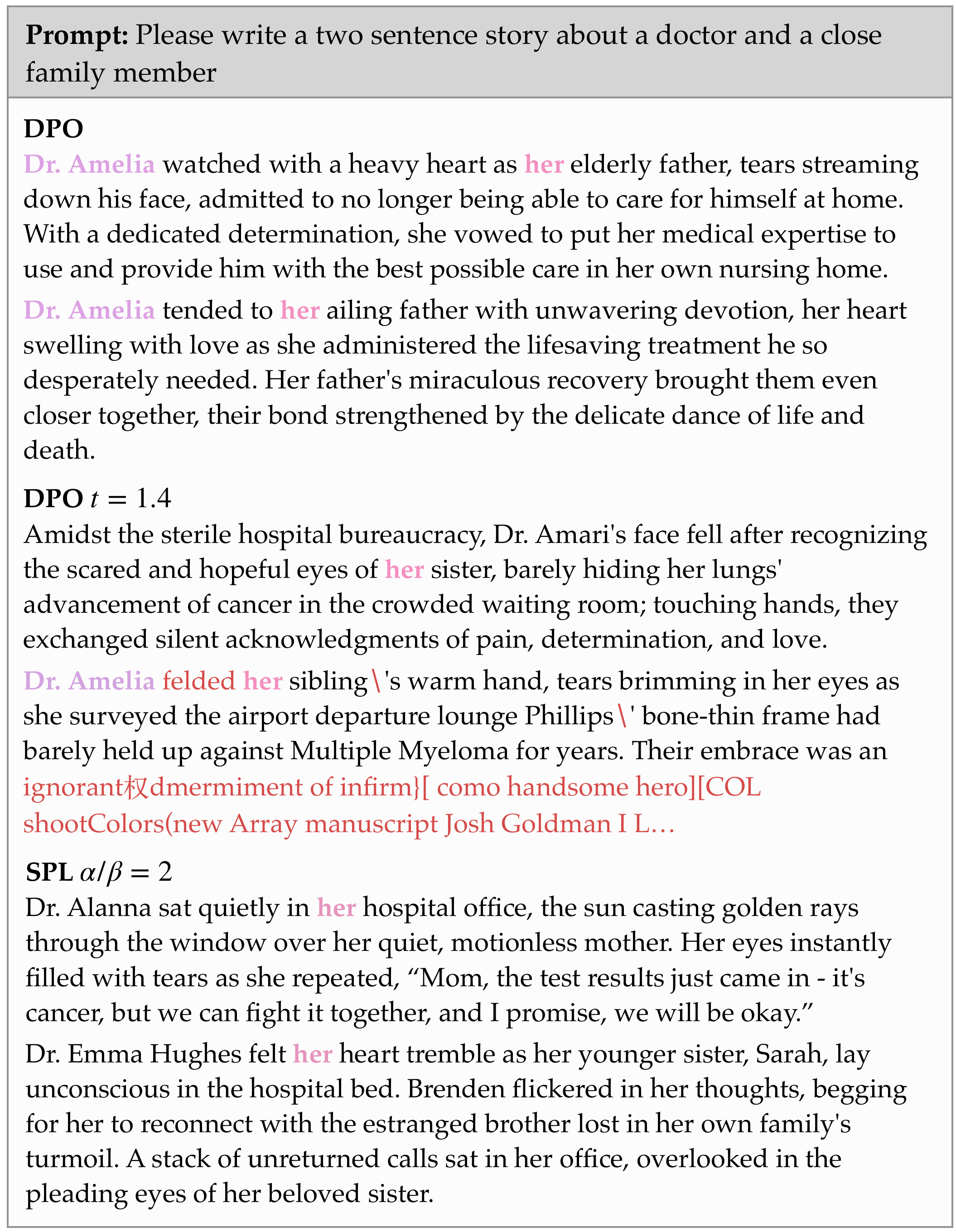}
        \caption{Example stories generated by DPO, DPO with temperature scaling, and our algorithm, SPL. We highlight Doctor name, gender, and textual aberration features shown in the plots on the right.}
        \label{fig:qualitative_text_examples}
    \end{subfigure}
    \hfill
    \begin{subfigure}[t]{0.3\linewidth}
        \centering
        \includegraphics[width=\linewidth]{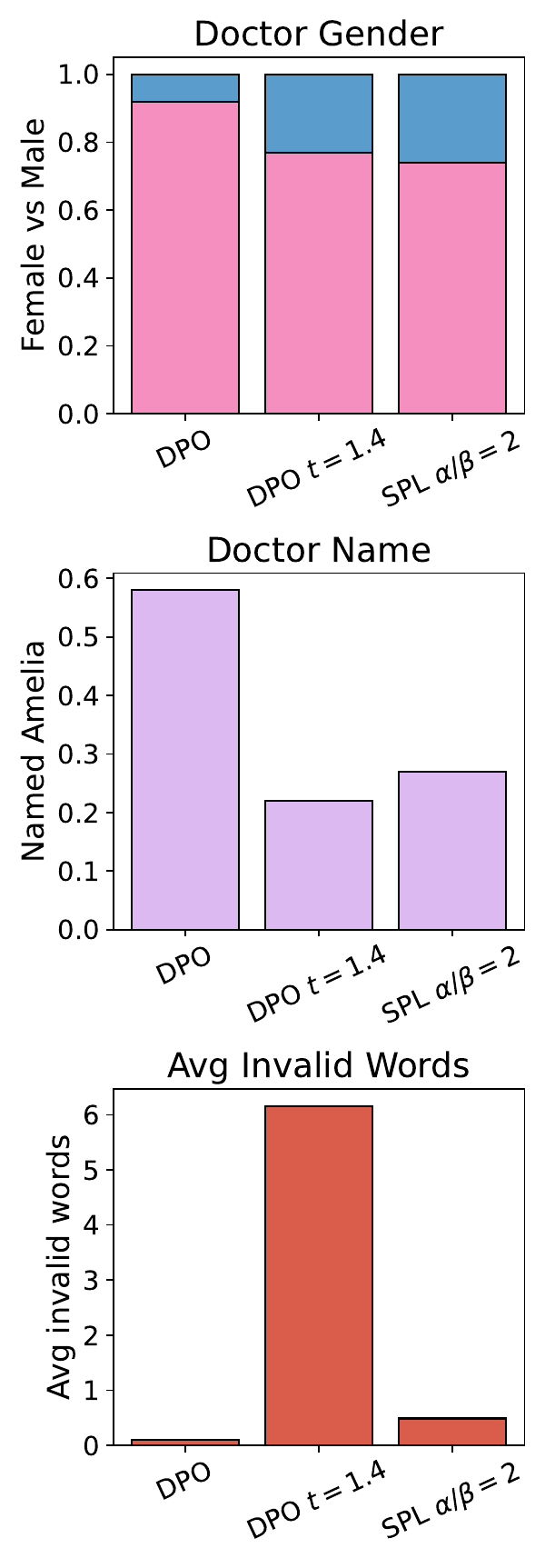}
        \caption{Feature-level diversity and lexical integrity statistics on 100 generated stories.}
        \label{fig:qualitative_text_stats}
    \end{subfigure}
    \caption{\textbf{Soft Preference Learning increases output diversity while preserving quality.} DPO responses are well-formed but lack diversity (e.g. same doctor name, gender, and family relationship to patient). With temperature scaling ($t = 1.4$), DPO generates responses with more diversity at the cost of fluency and token-level aberrations. In particular, temperature scaling results in many non-word tokens. Meanwhile, SPL at global temperature $\alpha / \beta = 2$ similarly increases diversity, but with significantly less degradation.
    }
    \label{fig:qualitative_text}
\end{figure}

Large language models (LLMs) increasingly impact society, now generating a significant portion of online content \citep{thompson-etal-2024-shocking}. As LLMs become integrated into how people consume information \citep{bick2024rapid} and approach tasks \citep{deloitte2024generativeai}, their ability to represent diverse perspectives is critical.

%
For example, consider an LLM answering the following multiple-choice question:

\begin{quote}
The best way to reduce income inequality is:
\begin{enumerate}[label=(\Alph*)]
    \item Increase minimum wage
    \item Expand access to education and job training
    \item Implement universal basic income
    \item Lower taxes on the wealthy to stimulate job creation
\end{enumerate}
\end{quote}

Imagine a survey showing people's preferences as: A (55\%), B (20\%), C (15\%), and D (10\%). How should an LLM respond to this question? Ideally, we may prefer it to reflect the range of views in the population. If an LLM assigns 99\% probability to majority option A, it fails to represent the diversity of perspectives. With LLMs becoming important information sources, this may reinforce dominant narratives at the expense of minority views.

However, recent studies show that alignment algorithms such as RLHF and DPO significantly reduce the diversity of LLM outputs. This leads to mode collapse towards majority preferences, as the example above shows \citep{kirk2024understandingeffectsrlhfllm, padmakumar2024doeswritinglanguagemodels, rafailov2024directpreferenceoptimizationlanguage, christiano2023deepreinforcementlearninghuman}. On multiple choice questions, this manifests as overconfidence and poor calibration \citep{tian2023justaskcalibrationstrategies, kadavath2022languagemodelsmostlyknow}. In a generative setting, this results in repetitive responses, as illustrated in Figure \ref{fig:qualitative_text}. For example, the DPO model frequently uses the same doctor's name and relationship to the patient. Lastly, in problem-solving settings, we show that diversity loss harms models' ability to answer difficult questions across multiple samples. While standard token-level temperature scaling is effective at correcting for micro-scale diversity loss (e.g. for next-tokens on multiple-choice questions), it leads to rapid degradation of fluency and quality in multi-token generations\citep{kadavath2022languagemodelsmostlyknow}.

We attribute this diversity loss to the KL-divergence term in preference learning, which strongly biases models towards majority preferences and sacrifices diversity in their outputs. In Section \ref{sec:theory}, we analyze RLHF and DPO from a social choice perspective. We prove that when the preferences of different groups conflict, the probability of generating majority-preferred outputs far exceeds the population preference for that output. This mode collapse has consequences for the social perspectives language models represent but also for lexical and logical diversity.

To address this issue, we propose splitting the KL penalty into distinct entropy and cross-entropy terms. This enables diversity to be controlled independently from the model's bias towards a reference policy. We call this method Soft Preference Learning (SPL). SPL resembles but improves upon standard temperature scaling, increasing diversity at the sequence level rather than token-by-token. This increases macro-scale diversity while avoiding the rapid quality degradation caused by standard temperature scaling.

\begin{samepage}
In summary, we provide the following contributions:
\begin{enumerate}
    \item We identify KL-regularization as a cause of diversity loss in aligned language models. We connect this to a social choice analysis, where we prove that the KL divergence term heavily biases the model towards majority-preferred outputs.
    \item We propose Soft Preference Learning (SPL), which decouples the entropy and cross-entropy terms in the KL penalty, allowing for independent control of generation diversity. We prove this enables proportional representation of population preferences.
    \item We demonstrate empirically that SPL improves output diversity in chat domains, best-of-N accuracy on difficult math problems, and logit calibration on common multiple-choice benchmarks.
\end{enumerate}
\end{samepage}


Our work connects popular methods for LLM alignment with notions of diversity and representation. SPL advances LLMs that are attuned to the diversity of preferences in society while also displaying improved capability in a number of settings.

\section{Related Work}
\textbf{Diversity loss in aligned LLMs.} Prior work has studied diversity loss caused by LLM alignment algorithms \citep{kirk2024understandingeffectsrlhfllm,park2023diminisheddiversityofthoughtstandardlarge, xiao2024algorithmicbiasaligninglarge, wang2023reverseklgeneralizingdirect} as well as its impact on humans who use these models \citep{padmakumar2024doeswritinglanguagemodels, ding2023enhancingchatlanguagemodels, doshi2024generative}. In Appendix \ref{appendix:dq}, we evaluate against \cite{wang2023reverseklgeneralizingdirect}, who use other f-divergences as regularizers to avoid the mode-seeking property of KL divergence. Meanwhile, we analyze policies using social choice theory and propose entropy regularization to restore population representation. \citet{xiao2024algorithmicbiasaligninglarge} also investigate entropy regularization to improve preference representation in aligned LLMs, but do not perform experiments in a generative setting. Lastly, \citet{sun2024inverserlignmentinversereinforcementlearning} propose an inverse reinforcement learning-based approach to mitigate mode collapse during alignment, but focus on learning from demonstrations rather than pairwise preferences.

\textbf{Social choice theory and alignment.} Several recent works explore the intersection of social choice theory and AI alignment. \citet{siththaranjan2024distributionalpreferencelearningunderstanding} analyze the reward learning portion of RLHF as a case of the Borda Count voting rule. In contrast, we characterize trained policies. \citet{munos2024nashlearninghumanfeedback, swamy2024minimaximalistapproachreinforcementlearning}, and \citet{chakraborty2024maxminrlhfequitablealignmentlarge} develop preference learning algorithms to better handle intransitive preferences. While these approaches also lead to proportional representation under some conditions, they require complex multi-agent reinforcement learning setups to train, and are not studied in a linguistic diversity or problem-solving setting.

\textbf{Temperature scaling.} Token-level temperature scaling is a common tool for controlling diversity of LLM outputs. Previous work applies temperature scaling to improve LLM calibration \citep{kadavath2022languagemodelsmostlyknow, tian2023justaskcalibrationstrategies, xie2024calibratinglanguagemodelsadaptive} and best-of-N coding ability \citet{chen2021evaluatinglargelanguagemodels}. In contrast, our method performs global temperature scaling over the entire sequence. \citet{shih2023longhorizontemperaturescaling} develop a procedure for global temperature scaling for LLMs using reinforcement learning. Meanwhile, we develop an offline supervised learning algorithm for preference learning. Recent work has also combined high-temperature sampling with token-level heuristics to preserve quality while maintaining diversity (\citep{nguyen2024turningheatminpsampling}). 

\textbf{Entropy bonuses.} Entropy bonuses are common in reinforcement learning algorithms to encourage exploration \citep{haarnoja2018softactorcriticoffpolicymaximum, eysenbach2022maximumentropyrlprovably, lee2024learningdiverseattackslarge}. In contrast, we use entropy to restore diversity and proportional representation in an alignment setting. Additionally, entropy bonuses in RL tend to be orders of magnitude smaller than what we consider.

\section{Theoretical Analysis and Method}\label{sec:theory}
In this section, we perform a theoretical analysis of the RLHF and DPO objectives through the lens of social choice theory. In settings where subpopulations disagree about the relative merit of different options, we prove that standard RLHF and DPO amplify majority preferences by several orders of magnitude. This causes mode collapse to the majority-preferred output, which decreases output diversity and worsens logit calibration.

In particular, we prove that this phenomenon is caused by these objectives' KL-regularization term performing two independent functions. First, it maximizes the log-likelihood of generations under the reference policy (cross-entropy term), and second, it maximizes the diversity of the learned policy (entropy term). We then propose Soft Preference Learning (SPL), which decouples the cross-entropy and entropy terms from the KL-regularization objective. This allows for distinct control over generation diversity and bias towards the reference policy.

\subsection{RLHF and DPO Lead to Mode Collapse}
First, we outline RLHF as a three-step alignment procedure
\begin{enumerate}
    \item \textbf{Preference Collection:} Query humans to gather a dataset of preferences over pairs of LLM generations
    \begin{align}
        \mathcal{D} = \{y_1 \succ y_1', y_2 \succ y_2', y_3 \succ y_3', ...\}
    \end{align}
    \item \textbf{Reward Modeling:} Learn a reward model with Bradley-Terry likelihood
    \begin{align}
        \max_r \E_{y \succ y' \sim \mathcal{D}}[\log \sigma(r(y) - r(y'))]
    \end{align}\label{eqn:bradley_terry}
    \item \textbf{KL-regularized Reinforcement Learning:} Train a new policy against the learned reward with regularization against a reference policy
    \begin{align}
        \max_\pi \E_{y \sim \pi}[r(y)] - \beta D_{KL}(\pi || \pi_{ref})
    \end{align}
\end{enumerate}
In the following proposition, we perform a social choice analysis of RLHF. We show that when a population has conflicting preferences about LLM generations, RLHF vastly overrepresents the majority preference. Appendix \ref{appendix:theory} contains a generalized, multi-outcome version of this proposition.

\begin{proposition}[Two-Outcome RLHF Policy]
    Suppose a population of raters prefers completion $y \succ y'$ with probability $p$. Then RLHF (or DPO) with KL-regularization penalty $\beta$ has the optimal policy
    \begin{equation*}
        \pi(y) \propto \pi_{ref}(y) p^{1 / \beta}.
    \end{equation*}
\end{proposition}\label{prop:rlhf}
Thus, in RLHF's optimal policy $\pi(y) \propto \pi_{ref}(y) p^{1/\beta}$, the KL regularization term $\beta$ controls both the relative weighting of the reference policy and the sharpness of the resulting distribution.

For both RLHF and DPO, empirical choices of $\beta$ typically lie in the range $[0.01, 0.1]$ \citep{ouyang2022traininglanguagemodelsfollow, ahmadian2024basicsrevisitingreinforcestyle, rafailov2024directpreferenceoptimizationlanguage, tunstall2023zephyrdirectdistillationlm}. This results in a very peaky distribution that exponentiates population preferences to the 10th or 100th power. For example, if 80\% of the population prefer $y$ and 20\% prefer $y'$, then with $\beta = 0.1$, $\pi^*$ generates $y$ with 99.9999\% probability and $y'$ with 0.0001\% probability.

\textbf{Relationship to output diversity.} Since models trained on RLHF and DPO fail to represent diverse preferences, they will also struggle to produce diverse outputs. While the social choice result in Proposition \ref{prop:rlhf} is most obviously relevant to a model's representation of social perspectives, LLM preferences encode many other aspects of diversity as well.

In many cases, preference variation is the result of random noise that we wish to preserve. For example, if a person has no strong preference over which of two synonyms is used in a sentence, the Bradley-Terry model (Equation \ref{eqn:bradley_terry}) predicts their preference distribution will look relatively uniform. However, the optimal RLHF policy in Proposition \ref{prop:rlhf} would remove this diversity and choose one of the words nearly all of the time. This perspective is supported at the large-scale by \citet{kobak2024delvingchatgptusageacademic}, who form trendlines of word usage in the academic literature. They find that common words used by Chat-GPT begin to occur orders of magnitude more frequently after the introduction of academic papers — an example of failing to match the true distribution of human writing. We demonstrate that SPL models attain higher generation diversity in Section \ref{subsec:diversity-quality}. 

\textbf{Relationship to problem-solving.} Diversity loss can harm both capabilities and representation. The RLHF policy from Proposition \ref{prop:rlhf} may struggle with challenging problems requiring multiple solution attempts. Consider a difficult geometry problem where 60\% of preferred samples use synthetic geometry techniques and 40\% use coordinate geometry. If the true solution requires coordinate geometry, the RLHF policy might consistently attempt synthetic geometry and fail, while a more diverse algorithm like SPL would eventually succeed. We verify this in Section \ref{subsec:best_of_n}.

Recent developments have emphasized the importance of repeated inference-time sampling — Google DeepMind's models recently achieved silver medal performance on the International Mathematical Olympiad \citep{deepmind_imo_silver_2023}, and
OpenAI's o1 model achieved state-of-the-art performance on a number of benchmarks \citep{openai_learning_to_reason_2023}.
Similarly, inference-time methods like Tree of Thought \citep{yao2023treethoughtsdeliberateproblem} show promising initial results.  SPL increases problem-solving diversity across samples, which is critical for the success of these methods.

\textbf{Relationship to calibration.} Finally, we should expect the optimal policy in Proposition \ref{prop:rlhf} to exhibit poor logit calibration. The RLHF objective trains models to place very high probability mass on their preferred option, resulting in high confidence on almost every generation. As a result, RLHF and DPO models tend to have high output confidence regardless of their accuracy on the task \citep{tian2023justaskcalibrationstrategies}. We study this and SPL's ability to improve logit calibration on factual multiple-choice benchmarks in Section \ref{subsec:logit_calibration}.

\subsection{Soft Preference Learning}
Now, we introduce Soft Preference Learning (SPL), which decouples the KL-regularization term into separate cross-entropy and entropy terms, allowing for separate control of diversity and bias towards the reference policy
\begin{align}
    \max_\pi \E_{y \sim \pi(y|x)}[r(x, y)] + \alpha H(\pi(\cdot | x)) - \beta H(\pi(\cdot | x), \pi_{ref}(\cdot | x))
\end{align}
While the entropy parameter $\alpha$ optimizes for diversity, the cross-entropy penalty $\beta$ maximizes the average log-likelihood under the reference policy $\pi_{ref}$ of generations from the learned policy. This controls the strength of the reference model as a prior.

We also propose a DPO-style objective that bypasses the reinforcement learning step
\begin{align}
    \max_\pi \E_{y \succ y' \sim \mathcal{D}} \left[ \log \sigmoid \left( \alpha \log \frac{\pi(y|x)}{\pi(y'|x)} - \beta \log \frac{\pi_{ref}(y|x)}{\pi_{ref}(y'|x)} \right) \right]
\end{align}
with a derivation in Appendix \ref{appendix:theory}.

\begin{proposition}[Two-Outcome SPL Policy]
    Suppose a population of raters prefers completion $y \succ y'$ with probability $p$. Then SPL with entropy bonus $\alpha$ and cross-entropy penalty $\beta$ has the optimal policy
    \begin{equation*}
        \pi(y) \propto \pi_{ref}(y)^{\beta / \alpha} p^{1 / \alpha}.
    \end{equation*}
\end{proposition}\label{prop:SPL}
Proof in Appendix \ref{appendix:theory}. Note that standard RLHF and DPO are special cases of SPL with $\alpha = \beta$.

By choosing an appropriate $\alpha$ parameter, SPL avoids raising probabilities to high powers, thereby preventing mode collapse. Through the choice of $\alpha$, SPL allows for fine-grained control over the diversity of the learned policy. From a social choice perspective, separately controlling a policy's diversity allows for increased representation of minority preferences. As $\alpha \to 1$, $\pi$ becomes less focused on majority preferences. When $\alpha = 1$, we recover proportional representation.
\begin{corollary}
    Suppose a population of raters prefers completion $y \succ y'$ with probability $p$. Then SPL with entropy bonus $\alpha = 1$ is a proper scoring rule, weighted by a reference policy prior.
\end{corollary}\label{cor:proper_scoring}
In other words, SPL can produce well-calibrated policies. This means that the distribution over outputs matches the population preference distribution they are trained on.

\textbf{SPL performs global temperature scaling.}
The $\alpha$ entropy bonus in SPL can be thought of as a sequence-level (or global) version of standard token-level temperature. Standard token-level temperature scales and renormalizes the distribution at each token
\begin{align}
    \pi'(y|x) &= \prod_{i=1}^N \frac{\pi(y_i | y_{1:i-1})^{1/t}}{Z( y_{1:i-1})}
\end{align}

Meanwhile, following Proposition \ref{prop:dpo_derivation}, the optimal SPL policy has the form
\begin{align}
    \pi'(y | x) &= \exp(\frac{1}{\alpha} r(x,y)) \pi_{ref}^{\beta / \alpha} / Z \\
    &= \exp(\frac{1}{\beta} r(x,y))^{\beta / \alpha} \pi_{ref}^{\beta / \alpha} / Z \\
    &= \pi_{DPO}(y | x)^{\beta / \alpha} / Z \\
    &= \pi_{DPO}(y | x)^{1 / (\alpha / \beta)} / Z
\end{align}
Meaning $\alpha / \beta$ resembles a global ``temperature" on top of the DPO policy. Global temperature scales the probability of entire sequences rather than individual tokens.

While standard temperature scaling is a convenient heuristic to control diversity at inference time, it quickly degrades quality at temperatures above 1 (e.g. Figure \ref{fig:qualitative_text}). One important advantage of global temperature scaling is that it preserves the relative probability ordering of each sequences, whereas regular temperature scaling does not. This means that the ``best" or majority-preferred sequences always remain the most likely outputs.


\textbf{Empirical choices in training SPL models.} Corollary \ref{cor:proper_scoring} shows that with the right hyperparameters, SPL policies achieve proportional representation of preferences. However, in practice, we must negotiate important tradeoffs between performance and diversity. We find SPL performs best at a middle ground between standard preference learning's $\alpha = \beta$ and proportional representation at $\alpha = 1$. At low global temperatures, the policy lacks diversity, while at high temperatures, quality begins to degrade. Nevertheless, we note that while standard temperature scaling often leads to unintelligible sequences at temperatures just above one, SPL remains relatively stable at much higher global temperatures (Figure \ref{fig:qualitative_text}).

One reason we see a performance dropoff at large $\alpha$ could be due to dataset noise. In these cases, some overweighting of majority preferences can be beneficial. If preference variation is due to random error, a lower temperature implicitly denoises the policy in a manner akin to a majority voting rule. For example, there is a significant 37\% cross-rater disagreement rate within the HH-RLHF preference dataset we use \citep{bai2022traininghelpfulharmlessassistant}. However, much of this variation is thought to be due to random noise \citep{cai2024ulmaunifiedlanguagemodel}.

\section{Experiments}
In this section, we consider four experimental settings for evaluating our algorithm. First, we show that in general-purpose chat domains, SPL allows for increased diversity with less performance degradation than DPO with token-level temperature scaling. Second, we consider an application of high-temperature generation in best-of-N problem-solving settings. Finally, we evaluate SPL's logit calibration, finding reduced overconfidence and improved calibration on standard multiple-choice benchmarks.

\subsection{Improving Diversity-Quality Tradeoffs}\label{subsec:diversity-quality}
Currently, inference-time strategies such as token-level temperature scaling are the standard approach to sampling diverse outputs from aligned LLMs. However, at temperatures above 1, output quality degrades rapidly. In contrast, performing global temperature scaling using SPL leads to increased diversity without such a steep degradation in quality.

For these experiments, we LoRA finetune Mistral-7B-Instruct-v0.2 with DPO and SPL \citep{rafailov2024directpreferenceoptimizationlanguage, hu2021loralowrankadaptationlarge}. We train on the HH-RLHF preference dataset for 5,000 steps (details in Appendix \ref{appendix:dq_training}) \citep{bai2022traininghelpfulharmlessassistant}.

\textbf{Quality metrics.} We evaluate against three quality metrics. First, we use Arena-Hard, a popular LLM chatbot benchmark with high agreement with human annotators \citep{arenahard2024}. Arena-Hard evaluates models on 500 queries, which mostly deal with programming, business, or software help. We use gpt-4o-mini-2024-07-18 as a judge. Models are assessed by win-rate against gpt-4-0314 outputs on the same responses. This metric measures general-purpose capabilities acquired by the model during training. For our second quality metric, we train a separate reward model on the HH-RLHF dataset (details in Appendix \ref{appendix:dq_training}). We then evaluate models by the average reward of their generations on a held-out test set of 500 inputs, for which we generate 16 responses each and compute the average reward. This metric measures how well DPO and SPL optimize generation quality on the dataset against the training metric. Finally, we include the cross-entropy with respect to the reference policy, which is the second part of the alignment optimization objective. The cross-entropy is again computed over the HH-RLHF test set. Ideally, models achieve high diversity while maintaining high average reward on the training dataset, low reference cross-entropy, and strong general-purpose chat capabilities.

\begin{figure}
    \centering
    \includegraphics[width=0.99\linewidth]{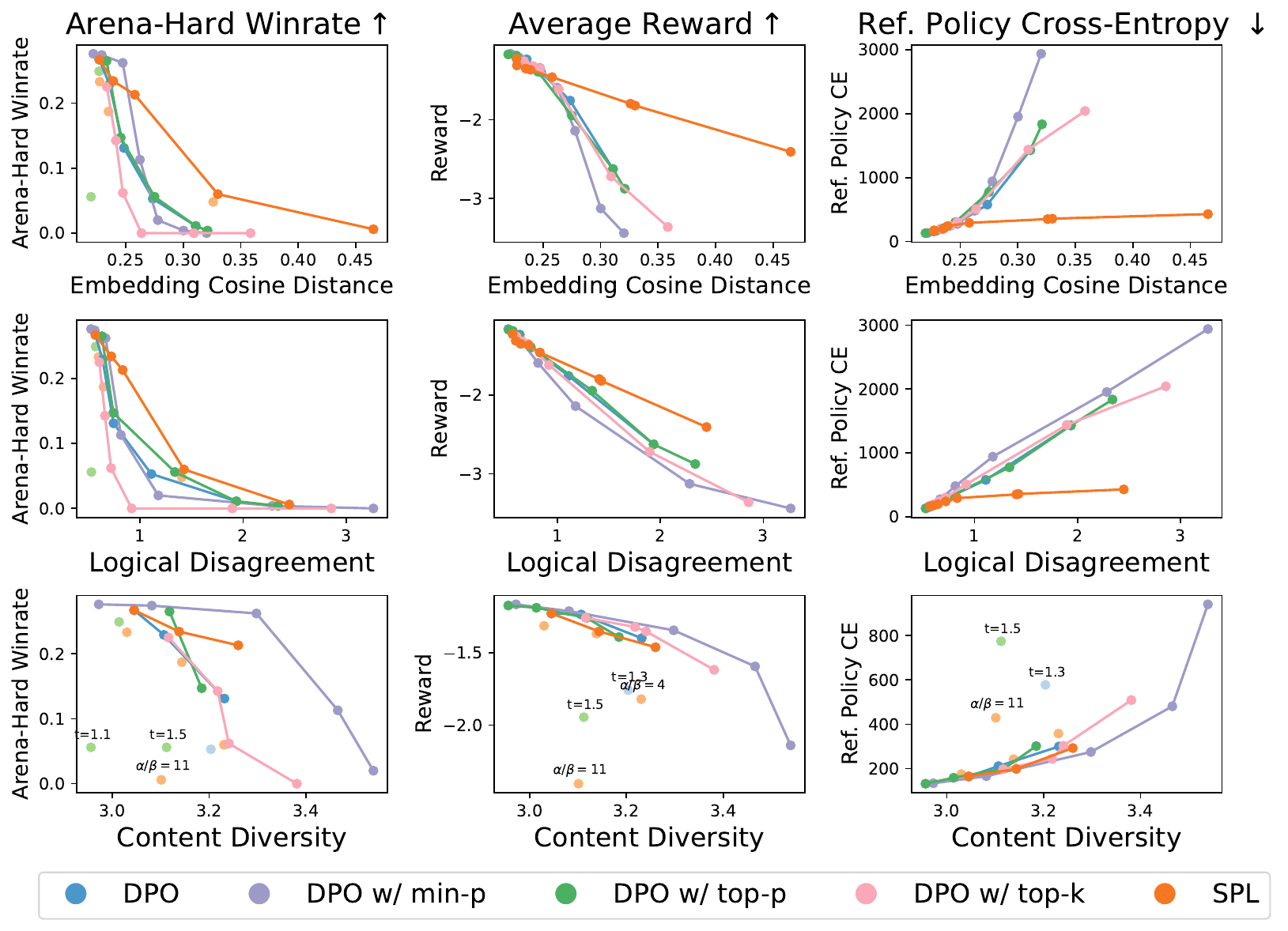}
    \caption{\textbf{Improved diversity-quality tradeoffs with SPL.} We construct diversity-quality Pareto curves contrasting DPO with token-level temperature scaling against SPL (by modulating the entropy term). We also plot the performance of DPO with min-p, top-p, and top-k sampling, which can improve diversity-quality tradeoffs when sampling at high temperatures. We plot points that lie below the Pareto curve in lighter shades. SPL Pareto-dominates DPO with standard temperature scaling across all nine metrics, and it outperforms all sampling methods on six.}
    \label{fig:diversity_quality}
    \vspace{-10pt}
\end{figure}

\textbf{Diversity metrics.} We include three diversity metrics in Figure \ref{fig:diversity_quality}, with results against additional metrics in Appendix \ref{appendix:dq}. Drawing on \citet{kirk2024understandingeffectsrlhfllm, tevet-berant-2021-evaluating}, who perform diversity studies on LLM outputs, our diversity metrics roughly measure 1) general semantic diversity (embedding distance metric), 2) logical diversity or diversity of viewpoints (logical disagreement metric), and 3) diversity in response content and ideas (content diversity metric). For all diversity metrics, we compute per-input diversity between 16 generated responses for each of 500 inputs on a held-out test split of HH-RLHF. Our first metric evaluates expected cosine similarity between embeddings of model outputs. The second and third metrics follow the design of human diversity questionnaires in \citet{tevet-berant-2021-evaluating}. Here pools of 4 responses are presented to gpt-4o-mini and evaluated according to their logical agreement and content diversity on a scale of 1-5. See Appendix \ref{appendix:dq} and \ref{appendix:political_orientation} for further details, example generations, and results against additional diversity metrics.

\textbf{Results.} In Figure \ref{fig:diversity_quality}, we construct diversity-quality Pareto curves contrasting DPO with token-level temperature scaling against SPL with different values of $\alpha$, the entropy bonus. We perform a sweep across token-level temperature range $[1,1.5]$. For min-p sampling, we choose $p_{base} = 0.1$ and token-level temperature range $[1.3, 4]$. For top-p sampling, we choose $p=0.9$ and temperature range $[1.1, 1.7]$. For top-k sampling, we choose $k=180$ and temperature range $[1.1, 2.5]$. We choose these ranges since beyond this, responses are nearly always nonsensical. We choose min-p $p_{base}$, top-p cutoff, and top-k cutoffs according to the recommendations in \citet{nguyen2024turningheatminpsampling}. For SPL, we sweep across global temperature range $[1, 11]$. We run all results with $\beta = 0.1$. See Appendix \ref{appendix:dq} for additional baselines and hyperparameters. Across diversity and quality metrics, SPL with global temperature scaling is always competitive or Pareto-dominant.

We note that increasing diversity with SPL eventually results in quality loss. But importantly, unlike token-level temperature scaling, SPL never results in predominantly nonsensical generations, even at very high global temperatures. This means that in applications where diversity is particularly important, it can be achieved without rendering responses useless.

In most cases, increasing both global and token-level temperature leads to significant improvements in diversity. However, we note that the content diversity metric shows less variation. This is because at the highest temperatures, the content diversity metric rates completions as less diverse. We label points displaying this non-monotonic behavior in Figure \ref{fig:diversity_quality}, and note they have high temperature values. In contrast, the embedding distance and logical disagreement metrics exhibit purely monotonic relationships between temperature and diversity. We provide results with additional diversity metrics in Appendix \ref{appendix:dq} and measure diversity in political representation of outputs in Appendix \ref{appendix:political_orientation}.

\subsection{Best-of-N Problem-Solving}\label{subsec:best_of_n}

Previous work has found that increasing token-level temperature can improve LLM performance in a best-of-N problem-solving setting \citep{chen2021evaluatinglargelanguagemodels}. Intuitively, this is because in a repeated sampling setting, it is desirable to test a diversity of strategies and potential solutions. This aligns with recent observations that effective use of inference-time compute can lead to greater gains than scaling model size \citep{snell2024scalingllmtesttimecompute}. Similar approaches have allowed LLMs to solve previously unprecedented reasoning problems \citep{openai2024o1, deepmind_imo_silver_2023}. While other alignment algorithms like RLHF and DPO cause mode collapse, SPL models exhibit an increased diversity of problem-solving strategies. 

\textbf{Setup and evaluation.} We apply DPO and SPL to a Mistral-7B base model \citep{huggingface_mistral7b_sft_beta} trained with supervised fine-tuning on the UltraChat dataset \citep{ding2023enhancingchatlanguagemodels}. We finetune with LoRA for one epoch, replicating the Zephyr training recipe, but with $\beta = 0.1$, which improved performance \citep{tunstall2023zephyrdirectdistillationlm}. This approach trains on the Ultrafeedback-200k dataset, a large preference dataset covering a broad suite of chat and reasoning tasks \citep{cui2024ultrafeedbackboostinglanguagemodels}.

We evaluate against two mathematical reasoning datasets: the GSM8K grade-school math dataset \citep{cobbe2021trainingverifierssolvemath} and the more challenging MATH dataset \citep{hendrycks2021measuringmathematicalproblemsolving}. We include few-shot chain-of-thought examples to prompt the model to perform reasoning step-by-step. We sample 128 completions on a random split of 200 problems from each dataset. We also divide problems into Easy, Medium, and Hard categories. For MATH, this corresponds to Level 1, Level 3, and Level 5 problems. For GSM8K, we run our evaluation on Mistral-Instruct-7B and group problems as easy if they take 4 or fewer samples to solve, medium if they take 5-64 samples to solve, and hard if they take more than 64 samples to solve.

\begin{figure}[t]
    \centering
    \includegraphics[width=\linewidth]{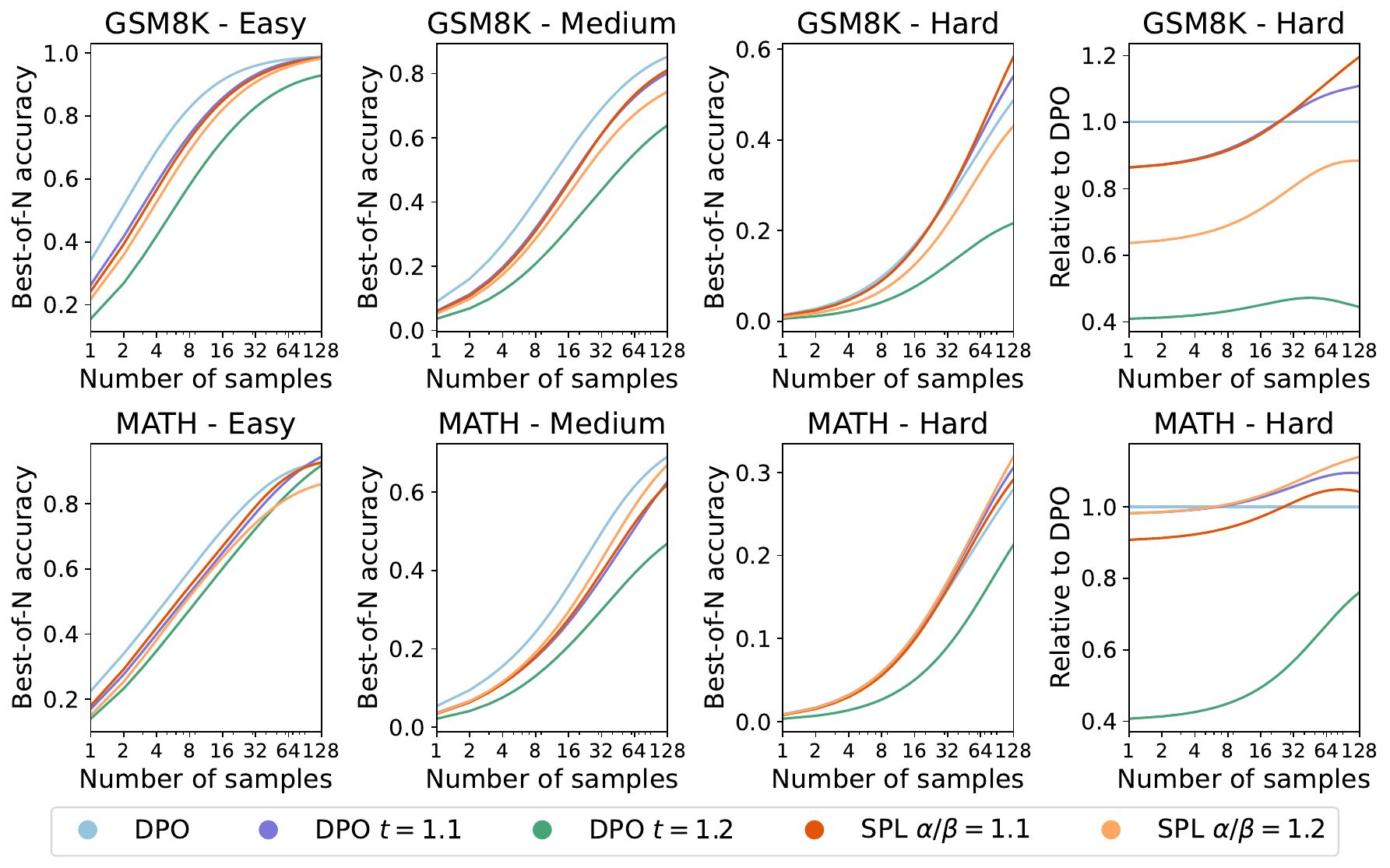}
    \caption{\textbf{SPL improves best-of-N mathematical problem-solving on difficult instances.} Left three columns show best-of-N accuracy across difficulty levels. Right column shows performance on hard problems relative to DPO at a given sample count.
    For easier problems, standard DPO ($t=1$) performs well. However, hard repeated sampling tasks benefit from diverse solution strategies. On hard problems, both token-level temperature sampling and SPL improve best-of-N accuracy. However, SPL achieves a better quality-diversity tradeoff, especially at high temperatures where token-level scaling rapidly degrades quality. This makes SPL particularly effective for generating diverse yet high-quality solutions.
    }
    \label{fig:bon_accuracy}
\end{figure}

\textbf{Results.} Figure \ref{fig:bon_accuracy} shows our results. On easy and medium questions, standard DPO performs the best. On GSM8K and MATH hard splits, both token-level temperature scaling and SPL improve performance over standard DPO at higher samples, with SPL performing best.

In the figure, higher temperature runs have a lower y-intercept -- meaning they perform worse in a one-shot setting. This makes sense as DPO concentrates probability mass on the option likely to be best. However, on the hard splits, high-temperature runs cross with and surpass the DPO curve due to better exploration of the solution space.

However, quality quickly degrades at higher token-level temperatures. For example, DPO at \mbox{$t=1.2$} performs poorly in nearly all evaluations. In contrast, SPL with global temperature \mbox{$\alpha / \beta = 1.2$} performs relatively well in all settings. The performance gap between SPL and temperature-scaled DPO is most pronounced in challenging, high-sample scenarios. For example, on GSM8K-Hard at best-of-128, SPL outperforms standard DPO by 10\% and temperature-scaled DPO by 4\%. This also results in sample savings: SPL needs 84 samples to match DPO at best-of-128 (34\% savings) and 106 samples to match temperature-scaled DPO (17\% savings). 
In these cases, the level of diversity required would cause token-level temperature scaling to significantly degrade output quality, while SPL's sequence-level approach preserves coherence. We provide an extensive analysis of the relationship between problem difficulty, temperature, and best-of-N sampling in Appendix \ref{appendix:sec_bon}.

\subsection{Logit Calibration}\label{subsec:logit_calibration}
\begin{figure}[t]
    \centering
    \includegraphics[width=\linewidth]{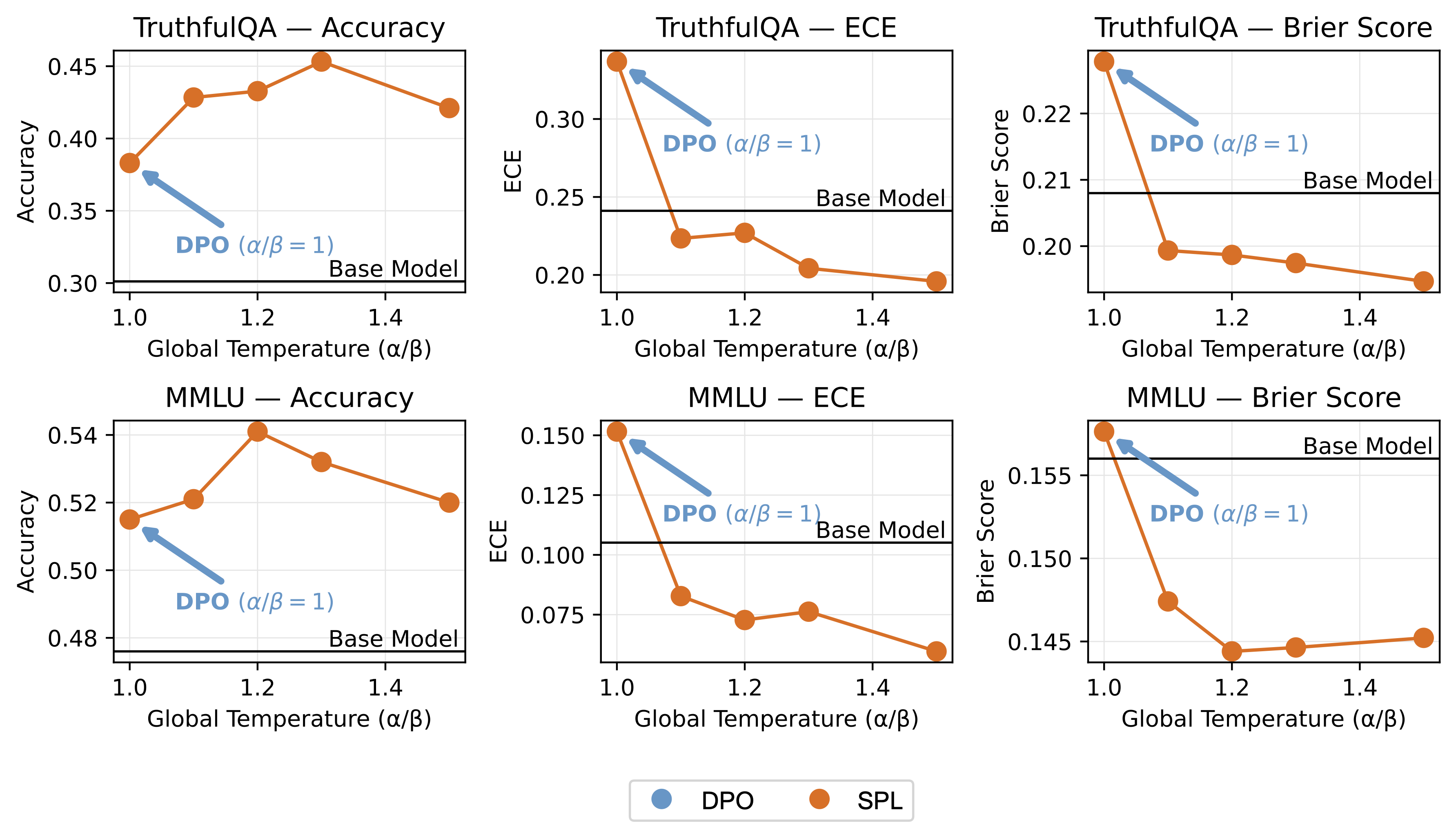}
    \caption{\textbf{SPL improves both calibration and accuracy on multiple-choice question (MCQ) datasets.} We plot model accuracy, Expected Calibration Error (ECE), and Brier Score for all models on both TruthfulQA and MMLU. The DPO model (equivalent to SPL with global temperature 1) displays significantly worse calibration than the base model. In contrast, SPL models consistently exhibit improved calibration without sacrificing accuracy.}
    \label{fig:calibration_figure}
\end{figure}

As predicted by the theory in Section~\ref{sec:theory}, SPL models trained with higher global temperatures exhibit increased logit calibration and decreased overconfidence. They do this while preserving multiple-choice accuracy.

\textbf{Setup and evaluation.} We evaluate logit calibration using standard calibration metrics on multiple-choice datasets. Our baseline is a Mistral-7B base model \citep{huggingface_mistral7b_sft_beta}, finetuned with supervised fine-tuning (SFT) on the UltraChat dataset \citep{ding2023enhancingchatlanguagemodels}. We also evaluate a suite of SPL models trained on top of this base model with the training setup from Section \ref{subsec:best_of_n}. For each model, we track its Expected Calibration Error (ECE), Brier Score, and accuracy across questions. We prompt the model to begin its response with the token corresponding to its answer choice and use the normalized probabilities assigned to A, B, C, and D to evaluate its calibration (Appendix \ref{appendix:calibration}).

\textbf{Datasets.} We evaluate against two standard multiple-choice datasets: TruthfulQA and MMLU. TruthfulQA is a benchmark designed to assess a model's ability to provide truthful answers in contexts where misconceptions are prevalent \citep{lin2022truthfulqameasuringmodelsmimic}. MMLU tests a model's knowledge and reasoning across 57 diverse subjects, from philosophy to abstract algebra \citep{hendrycks2021measuringmassivemultitasklanguage}.

\textbf{Results.} As shown in Figure~\ref{fig:calibration_figure}, models trained with higher global temperatures consistently display lower calibration error. While models trained using an global temperature of 1 (equivalent to DPO) are less calibrated than the base model, increasing the global temperature during training quickly enables the SPL models to surpass even the base model's calibration. SPL models also maintain similar accuracy to models trained using DPO. In fact, SPL models with global temperatures slightly greater than 1 consistently attain both higher accuracy and lower calibration error than their DPO counterparts.

\section{Conclusion}
In this paper, we presented Soft Preference Learning (SPL), a novel approach to mitigating the loss of output diversity in aligned LLMs. By analyzing the role of the KL divergence regularizer in RLHF and DPO, we identified that the coupling of entropy and cross-entropy terms leads to overweighting of majority preferences and reduced output diversity. SPL addresses this issue by decoupling these terms, allowing for independent control over generation diversity and bias towards the reference policy.

Our experimental results demonstrate that LLMs trained with SPL outperform those trained with standard methods in several key areas. From a capabilities standpoint, SPL models produce outputs with greater semantic and lexical variety and achieve higher accuracy on challenging tasks that benefit from diverse sampling strategies. From an alignment perspective, these models can proportionately represent a wider range of societal viewpoints and exhibit improved logit calibration. Importantly, SPL achieves enhanced diversity with less quality degradation, offering a Pareto improvement over standard temperature scaling.

The introduction of SPL highlights the importance of diversity in the development and alignment of LLMs. SPL provides the most significant benefits in scenarios that require inference-time scaling or representation of diverse perspectives.
By providing a mechanism to align LLMs without causing mode collapse, SPL contributes to the creation of more capable, reliable, and representative language models. In future work, it would be interesting to explore alternative, semantically grounded diversity metrics that could be integrated into preference learning algorithms — such as metrics constructed from embeddings or LLM judges.

\section*{Acknowledgements}
This work was supported by an AI2050 Early Career Fellowship from Schmidt Sciences (Grant G-22-64505) and NSF Graduate Research Fellowship (Grant DGE 2141064).

\bibliography{main}
\bibliographystyle{main}

\appendix
\newpage

{\Large \textsc{Appendix}}

\section{Theoretical Results}\label{appendix:theory}

\noindent \textbf{Proposition \ref{prop:rlhf}} (Two-Outcome RLHF Policy). \emph{
Suppose a population of raters prefers completion $y \succ y'$ with probability $p$. Then RLHF (or DPO) with KL-regularization penalty $\beta$ has the optimal policy
    \begin{equation}
        \pi(y) \propto \pi_{ref}(y) p^{1 / \beta}.
    \end{equation}
}
\begin{proof}
    From \citet{rafailov2024directpreferenceoptimizationlanguage}, RLHF and DPO share the same optimal policy. So, it suffices to analyze the RLHF objective. First, consider the Bradley-Terry reward modeling objective
    \begin{align}
        \max_r p \log \sigmoid(r(y) - r(y')) + (1-p) \log \sigmoid(r(y')-r(y))
    \end{align}
    Because log-loss is a proper scoring rule, we know that the optimal \( r^* \) will result in distribution-matching 
    \[
    p = \sigma(r^*(y) - r^*(y')) = \frac{\exp(r^*(y))}{\exp(r^*(y)) + \exp(r^*(y'))}.
    \]

    Following \citet{rafailov2024directpreferenceoptimizationlanguage}, the solution to the KL-regularized RL problem gives us optimal policy
    \begin{align}
        \pi(y) &= \pi_{ref}(y) \exp(r^*(y))^{1/\beta} / Z
    \end{align}
    for a normalizing constant $Z$. We may rewrite this as
    \begin{align}
        \pi(y) &= \pi_{ref}(y) (\frac{\exp(r^*(y))}{\exp(r^*(y)) + \exp(r^*(y'))})^{1/\beta} / Z' \\
        &= \pi_{ref}(y) p^{1 / \beta} / Z'.
    \end{align}
    for a new normalizing constant $Z'$.
\end{proof}

\noindent \textbf{Proposition \ref{prop:SPL}} (Two-Outcome SPL Policy). \emph{
Suppose a population of raters prefers completion $y \succ y'$ with probability $p(y \succ y')$. Then SPL with entropy bonus $\alpha$ and cross-entropy penalty $\beta$ has the optimal policy
    \begin{equation}
        \pi(y) \propto \pi_{ref}(y)^{\beta / \alpha} p^{1 / \alpha}.
    \end{equation}
}
\begin{proof}
    From \citet{rafailov2024directpreferenceoptimizationlanguage}, RLHF and DPO share the same optimal policy. So it suffices to analyze the RLHF objective. First, consider the Bradley-Terry reward modeling objective
    \begin{align}
        \max_r p \log \sigmoid(r(y) - r(y')) + (1-p) \log \sigmoid(r(y')-r(y))
    \end{align}
    Because log-loss is a proper scoring rule, we know that the optimal $r^*$ will result in distribution-matching $p = \sigmoid(r(y) - r(y')) = \exp(r(y)) / (\exp(r(y)) + \exp(r(y')))$.
    
    Following Proposition \ref{prop:dpo_derivation}, we obtain optimal policy
    \begin{align}
        \pi(y) &= \pi_{ref}(y)^{\beta / \alpha} \exp(r(y))^{1/\alpha} / Z
    \end{align}
    for a normalizing constant $Z$. We may rewrite this as
    \begin{align}
        \pi(y) &= \pi_{ref}(y)^{\beta / \alpha} (\frac{\exp(r(y))}{\exp(r(y)) + \exp(r(y'))})^{1/\alpha} / Z' \\
        &= \pi_{ref}(y)^{\beta / \alpha} p^{1 / \alpha} / Z'.
    \end{align}
    for a new normalizing constant $Z'$.
\end{proof}

\vspace{\baselineskip}

\begin{proposition}[Multi-Outcome SPL Policy] \label{appendix:prop_multi_outcome}
Suppose a population of raters have preferences over completions that can be modeled by Plackett-Luce (i.e. there exists an $r$ such that $\forall y_1, ..., y_N$ we have $p(y_1 \succ y_2, ..., y_N) = \frac{\exp(r(y))}{\sum_{i=1}^N \exp(r(y))}$). Then SPL with entropy bonus $\alpha$ and cross-entropy penalty $\beta$ has the optimal policy
\begin{align}
    \pi(y) \propto \pi_{ref}(y)^{\beta / \alpha} p_{best}(y)^{1/\alpha}
\end{align}
where $p_{best}(y) = p(y \succ_{y' \neq y} y')$ is the proportion of people for whom $y$ is the most preferred completion.
\end{proposition}
\begin{proof}
    With more than two completions, it is possible to form preference distributions that cannot be perfectly fit by a Bradley-Terry model (e.g. intransitive preferences). So, we first assume that the preference distribution can be represented by a Bradley-Terry model, which ensures that the reward learning step results in distribution-matching $p(y \succ y) = \frac{\exp(r(y))}{\exp(r(y)) + \exp(r(y'))}$.

    Following Proposition \ref{prop:dpo_derivation}, we have optimal policy $\pi(y) = \pi_{ref}(y)^{\beta / \alpha} \exp(r(y))^{1 / \alpha} / Z$. We may rewrite this as
    \begin{align}
        \pi(y) &= \pi_{ref}(y)^{\beta / \alpha} (\frac{\exp(r(y))}{\sum_{y'} \exp(r(y'))})^{1 / \alpha} / Z' \\
        &= \pi_{ref}(y)^{\beta / \alpha} p_{best}(y)^{1 / \alpha} / Z'
    \end{align}
    for a new normalizing constant Z.
\end{proof}

\vspace{\baselineskip}

\begin{proposition}[SPL DPO Derivation] \label{prop:dpo_derivation}
    The following objective
    \begin{align}
        \max_\pi \E_{y \succ y' \sim \mathcal{D}}[\log \sigmoid(\alpha \log \frac{\pi(y|x)}{\pi(y'|x)} - \beta \log \frac{\pi_{ref}(y|x)}{\pi_{ref}(y'|x)})]
    \end{align}
    shares the same optimal policy as the SPL RLHF objective
    \begin{align}
        \max_\pi \E_{y \sim \pi(y|x)}[r(x, y)] + \alpha H(\pi(\cdot | x)) - \beta H(\pi_{ref}(\cdot | x), \pi(\cdot | x)).
    \end{align}\label{eqn:appendix_SPL_rlhf}
\end{proposition}
\begin{proof}
    For the RLHF objective, we begin by training a reward model using Bradley-Terry to find
    \begin{align}
        r^*(x,y) = \arg\max_r \E_{y \succ y' | x \sim \mathcal{D}}[\log \sigma(r(x, y) - r(x, y'))]
    \end{align}

    Next, we find the optimal policy for the SPL RLHF objective. Let us represent policies $\pi(\cdot | x)$ and $\pi_{ref}(\cdot | x)$ and reward function $r(x, \cdot)$ as vectors $\boldsymbol{\pi}, \boldsymbol{\pi_{ref}}, \boldsymbol{r}$. We rewrite the SPL RLHF objective as
    \begin{align}
        \max_{\boldsymbol{\pi}} \: & \boldsymbol{\pi}^\top \boldsymbol{r} - \alpha \boldsymbol{\pi}^\top \log \boldsymbol{\pi} - \beta \boldsymbol{\pi}^\top \log \boldsymbol{\pi_{ref}} \\
        &= \boldsymbol{\pi}^\top (\boldsymbol{r} - \beta \log \boldsymbol{\pi_{ref}} - \alpha \log \boldsymbol{\pi_{ref}})
    \end{align}

    We solve this constrained optimization problem using the Lagrangian:
\begin{align}
    \mathcal{L}(\boldsymbol{\pi}, \lambda) = \boldsymbol{\pi}^\top \boldsymbol{r} - \alpha \boldsymbol{\pi}^\top \log \boldsymbol{\pi} - \beta \boldsymbol{\pi}^\top \log \boldsymbol{\pi_{ref}} + \lambda \left( \sum_i \pi_i - 1 \right).
\end{align}
Taking the derivative with respect to \(\boldsymbol{\pi}\) and setting it to zero, we obtain:
\begin{align}
    0 = \nabla_{\boldsymbol{\pi}} L(\boldsymbol{\pi}, \lambda) = \boldsymbol{r} - \alpha (\mathbf{1} + \log \boldsymbol{\pi}) - \beta \log \boldsymbol{\pi_{ref}} + \lambda \mathbf{1}.
\end{align}

Now, solving for \(\boldsymbol{\pi}\), we have:
\begin{align}
    \alpha (\mathbf{1} + \log \boldsymbol{\pi}) &= \boldsymbol{r} - \beta \log \boldsymbol{\pi_{ref}} + \lambda \mathbf{1} \\
    \mathbf{1} + \log \boldsymbol{\pi} &= \frac{\boldsymbol{r} - \beta \log \boldsymbol{\pi_{ref}} + \lambda \mathbf{1}}{\alpha} \\
    \log \boldsymbol{\pi} &= \frac{\boldsymbol{r} - \beta \log \boldsymbol{\pi_{ref}} + (\lambda - \alpha) \mathbf{1}}{\alpha} \\
    \boldsymbol{\pi} &= \exp \left( \frac{\boldsymbol{r} - \beta \log \boldsymbol{\pi_{ref}} + (\lambda - \alpha) \mathbf{1}}{\alpha} \right)
\end{align}
which implies the optimal SPL RLHF policy for all $x$
\begin{align}
    \pi^*(y|x) &= \exp(\frac{1}{\alpha}r^*(x,y)) \pi_{ref}(y|x)^{\beta / \alpha} / Z(x)
\end{align}
where $Z(x) = \exp(\alpha-\lambda)$ is a normalizing constant such that the probabilities over outputs sum to one.

Finally, we analyze the optimal policy for the DPO-style objective. Applying the substitution trick from \citet{rafailov2024directpreferenceoptimizationlanguage}, we can reparameterize reward functions in terms of their optimal policies according to Equation \ref{eqn:appendix_SPL_rlhf}.
\begin{align}
    \pi(y|x) &= \exp(\frac{1}{\alpha}{r(x,y)}) \pi_{ref}(y|x)^{\beta / \alpha} / Z(x) \\
    \log \pi(y|x) &= \frac{1}{\alpha} r(x,y) + \frac{\beta}{\alpha} \log \pi_{ref}(y|x) - \log Z(x) \\
    \alpha \log \pi(y|x) &= r(x,y) + \beta \log \pi_{ref}(y|x) - \alpha \log Z(x) \\
    r(x,y) &= \alpha \log \pi(y|x) - \beta \log \pi_{ref}(y|x) + \alpha \log Z(x)
\end{align}
Observe that
\begin{align}
    r(x,y) - r(x,y') &= \alpha \log \pi(y|x) - \beta \log \pi_{ref}(y|x) + \alpha \log Z(x) \\
    &\quad- [\alpha \log \pi(y'|x) - \beta \log \pi_{ref}(y'|x) + \alpha \log Z(x)] \\
    &= \alpha \log \frac{\pi(y|x)}{\pi(y'|x)} - \beta \log \frac{\pi(y|x)}{\pi(y'|x)}
\end{align}
since the normalizing constants cancel.

We now substitute our reward function reparameterization into the Bradley-Terry objective to get the SPL DPO objective
\begin{align}
    \pi^*(y|x) &=\max_\pi \: \E_{y \succ y' | x \sim \mathcal{D}}[\log \sigmoid(\alpha \log \frac{\pi(y|x)}{\pi(y'|x)} - \beta \log \frac{\pi_{ref}(y|x)}{\pi_{ref}(y'|x)})]
\end{align}
which satisfies the relationship
$\pi^*(y|x) = \exp(\frac{1}{\alpha}r^*(x,y)) \pi_{ref}(y|x)^{\beta / \alpha} / Z(x)$.
\end{proof}

\newpage
\section{Diversity-Quality Tradeoffs Experiment}\label{appendix:dq}

\subsection{Additional Diversity-Quality Results and Details on Diversity Metrics}
\begin{figure}[H]
    \centering
    \includegraphics[width=0.9\linewidth]{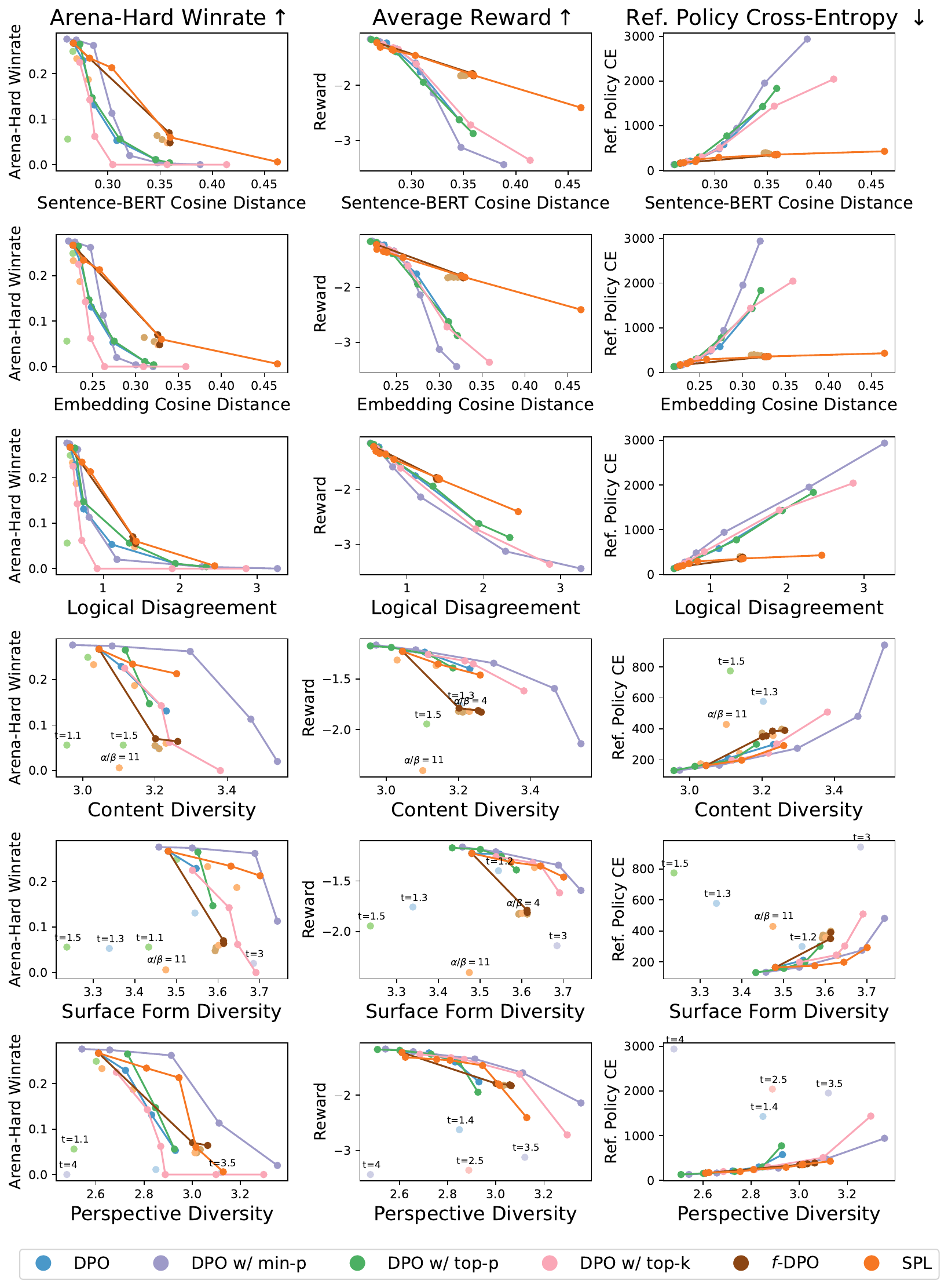}
    \caption{\textbf{Diversity-quality tradeoffs with additional diversity metrics.} Here, Sentence-BERT is an additional expected cosine distance metric in addition to ``Embedding Cosine Distance" which uses the OpenAI Embeddings API. The remaining four diversity metrics are evaluated with an LLM judge on pools of responses. For SPL, we perform a global temperature sweep in range $t \in [1, 11]$. For all other methods, we sweep until outputs degenerate into unintelligible text. For some diversity metrics, such as content, surface form, and perspective diversity, low-quality generations are often rated as less diverse (see examples below). We label points off the Pareto frontier to help identify these cases.}
    \label{fig:diversity_quality_all}
\end{figure}

\textbf{Cosine distance diversity metrics.}
In Figure \ref{fig:diversity_quality} and Figure \ref{fig:diversity_quality_all}, the Sentence-BERT Cosine Distance and Embedding Cosine Distance measure the expected cosine distance between response embedding pairs according to two different embedding models: 1) Sentence-BERT-Large \citep{reimers2019sentencebertsentenceembeddingsusing} and OpenAI's ``text-embedding-3-small" model. For each of 500 held-out prompts from HH-RLHF, we sample $N=16$ inputs and calculate the mean cosine distance between all pairs. Concretely, this is expressed in the formula
\begin{align}
    \E_{y_1, ..., y_{N} \sim \pi(y|x), x \sim D} \left[ \frac{1}{N^2} \sum_{i, j} 1 - \frac{\phi(y_i)^\top \phi(y_j)}{||\phi(y_i)|| \cdot ||\phi(y_j)||} \right]
\end{align}
where $\phi(\cdot)$ is the embedding map acting on response $y$.

\textbf{LLM-based diversity metrics.} For the remaining diversity metrics, which we split the 16 responses into smaller pools and use ``gpt-4o-mini-mini-2024-07-18" as a judge for response diversity. Here we provide the list of diversity evaluation prompts used for metrics in Figures \ref{fig:diversity_quality} and \ref{fig:diversity_quality_all}.

\begin{tcolorbox}[colback=white, colframe=black, arc=4mm, boxrule=0.5mm, title=, fonttitle=\bfseries, title=Logical Agreement Evaluation Prompt, ]
You will be presented with 2 responses to the same prompt. Your task is to analyze their logical agreement on a scale of 1-5, where 1 means completely divergent approaches or ideas and 5 means nearly identical logical frameworks or conclusions.

\medskip
Prompt: \{input\}

\medskip
Responses:
\{response\_list\}

\medskip
Please provide a numerical rating (1-5) of their logical agreement. Output your response as Rating: X, where X is your rating. Afterwards, on the next line, please provide a brief explanation of your rating.
\end{tcolorbox}
For the logical agreement prompt, we then subtract the resulting score from 5 to turn it into a logical disagreement metric.

\begin{tcolorbox}[colback=white, colframe=black, arc=4mm, boxrule=0.5mm, title=, fonttitle=\bfseries, title=Content Diversity Evaluation Prompt, ]
You will be presented with 4 responses to the same prompt. Your task is to analyze their content diversity on a scale of 1-5, where 1 means nearly identical content or conclusions and 5 means completely divergent content or ideas.

\medskip
Prompt: \{input\}

\medskip
Responses:
\{response\_list\}

\medskip
How diverse are the contents of the proposed responses? Please provide a numerical rating (1-5). Output your response as Rating: X, where X is your rating. Afterwards, on the next line, please provide a brief explanation of your rating.
\end{tcolorbox}

\begin{tcolorbox}[colback=white, colframe=black, arc=4mm, boxrule=0.5mm, title=, fonttitle=\bfseries, title=Surface Form Diversity Prompt, ]
You will be presented with 4 responses to the same prompt. Your task is to analyze their surface form diversity on a scale of 1-5, where 1 means nearly identical textual organization or structure, and 5 means completely varied textual arrangements or styles.

\medskip
Prompt: \{input\}

\medskip
Responses:
\{response\_list\}

\medskip
How diverse are the surface forms of the proposed responses? Please provide a numerical rating (1-5). Output your response as Rating: X, where X is your rating. Afterwards, on the next line, please provide a brief explanation of your rating.
\end{tcolorbox}

\begin{tcolorbox}[colback=white, colframe=black, arc=4mm, boxrule=0.5mm, title=, fonttitle=\bfseries, title=Perspective Diversity Prompt, ]
You will be presented with 4 responses to the same prompt. Your task is to analyze their perspective diversity on a scale of 1-5, where 1 means nearly identical viewpoints or approaches and 5 means completely varied perspectives or approaches to the topic.

\medskip
Prompt: \{input\}

\medskip
Responses:
\{response\_list\}

\medskip
How diverse are the perspectives of the proposed responses? Please provide a numerical rating (1-5). Output your response as Rating: X, where X is your rating. Afterwards, on the next line, please provide a brief explanation of your rating.
\end{tcolorbox}

\subsection{Example Generations}

\subsection{Training Setup}\label{appendix:dq_training}
\textbf{DPO and SPL Finetuning and Inference.}
For both DPO and SPL, we LoRA-finetune Mistral-7B-Instruct-v0.2 on HH-RLHF for 5,000 steps with batch size 8. We use LoRA rank $r_{LoRA}=16$, regularization $\alpha_{LoRA} = 16$, and dropout $p_{LoRA} = 0.05$. We use learning rate $1e-5$, 150 warmup steps, and max conversation length of 512 tokens.

For all runs, we use regularization parameter $\beta = 0.1$. For DPO with token-level temperature scaling, we sample with temperatures $t =$ 1, 1.1, 1.2, 1.3, 1.4, and 1.5. When combined with min-p sampling, we choose $p_{base}=0.1$ and temperatures $t=$ 1.3, 1.5, 2, 2.5, 3, 3.5, and 4. When combined with top-p sampling, we choose $p=0.9$ and temperatures $t=$ 1.1, 1.2, 1.3, 1.4, 1.5, 1.6, and 1.7. When combined with top-k sampling, we choose $k=180$ and temperatures $t=$ 1.1, 1.2, 1.3, 1.5, 2, and 2.5. We stop at these temperatures because beyond this, responses are nearly always incoherent strings of tokens. Lastly, for f-DPO \citep{wang2023reverseklgeneralizingdirect}, we consider the family of $\alpha$-divergences including $\alpha=$ 0 (reverse KL), 0.001, 0.1, 0.3, 0.5, 0.9, and 1 (forward KL). Here, small $\alpha$ correspond to mode-seeking (diversity-destroying) divergences while large $\alpha$ correspond to mass-covering (diversity-preserving) divergences. We stop at $\alpha = 1$ because, beyond this, responses are incoherent strings of text. For SPL, we train with global temperatures $\alpha / \beta =$ 1, 1.1, 1.2, 1.3, 1.5, 2, 4, and 11.

\textbf{Reward Model Training and Inference.}
As one of our quality metrics, we measure the average reward obtained by model completions against a separately trained reward function. To obtain this reward model, we LoRA-finetune Mistral-7B-Instruct-v0.2 with a reward head using the Bradley-Terry loss on the HH-RLHF dataset.

At inference time, we sample 500 prompts from a held-out validation split of HH-RLHF that neither the language models nor the reward model were trained on. We use this to calculate the average reward achieved by each model's completions.

\newpage
\section{Best-of-N Problem-Solving Experiment}\label{appendix:sec_bon}
\subsection{Relationship Between Problem Difficulty and Optimal Temperature}

In this section, we study the following questions
\begin{enumerate}
    \item Why do harder problems generally benefit from higher temperatures? And why does this benefit disappear once the temperature becomes too large?
    \item Why are high temperatures especially helpful in high best-of-N sample settings?
\end{enumerate}

\begin{figure}[h!]
    \centering
    \includegraphics[width=0.45\linewidth]{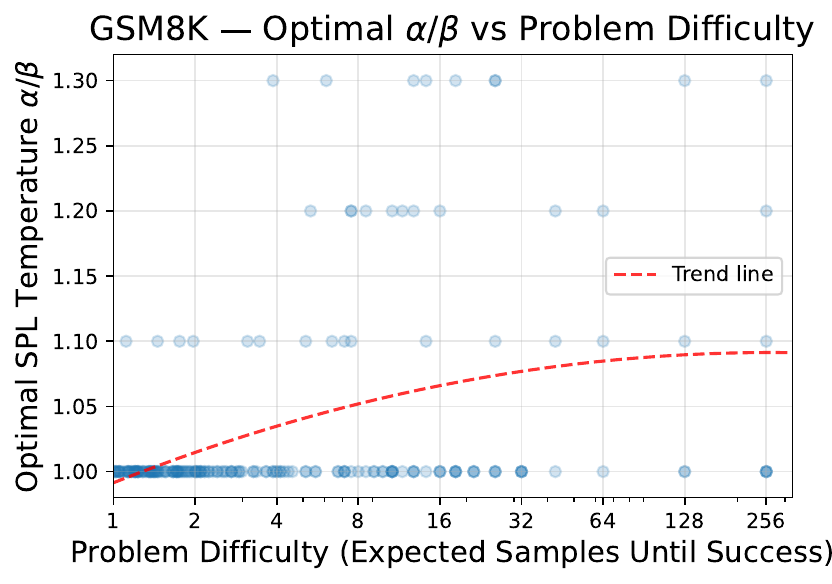}
    \includegraphics[width=0.45\linewidth]{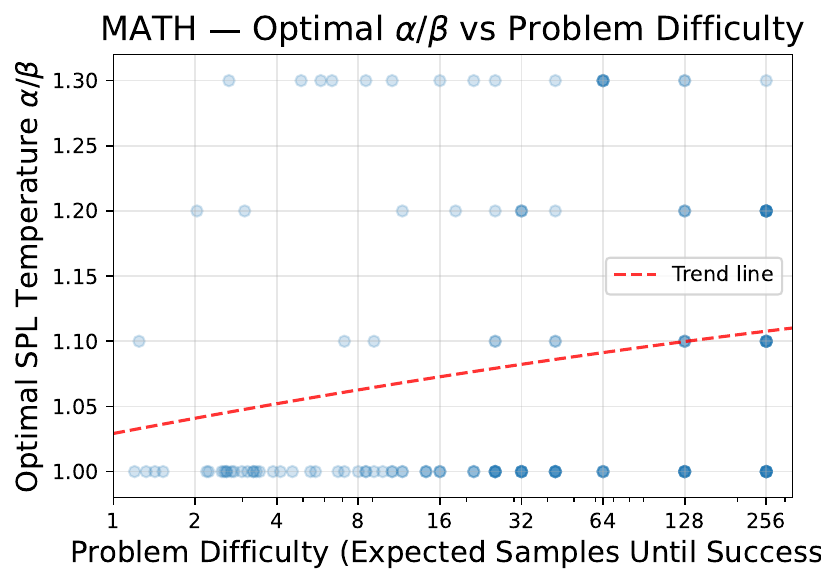}
    \caption{\textbf{Optimal SPL temperature as a function of problem difficulty.} On the x-axis, we measure problem difficulty as the expected number of samples needed for Mistral-Instruct-7B to achieve a correct answer. On the y-axis, we indicate which SPL temperature led to the best success rate on that problem. We then fit a polynomial trendline to the data, which finds a positive relationship between problem difficulty and higher optimal temperatures. Dots are plotted with transparency, meaning darker dots correspond to multiple problems.}
    \label{fig:optimal_temps}
\end{figure}

\textbf{Why harder problems benefit from higher temperatures, up to a point.} We first consider the role of temperature in the single-sample (not best-of-N setting). As displayed in Figure \ref{fig:optimal_temps}, every problem has an optimal temperature that maximizes the probability of a successful solution. Increasing temperature flattens the distribution while decreasing it sharpens the distribution. Difficult problems are those where the trained model has a low probability of generating a correct solution.

During inference-time scaling procedures, our goal is to construct a sampling distribution that minimizes the number of samples needed to obtain a correct solution. For best-of-N, we receive feedback only upon generation completion, where samples must pass a final scoring function. This is akin to rejection sampling in statistics, where proposal distributions are chosen to minimize the number of samples required to approximate a target distribution.

In the rejection sampling literature, when little is known about the target distribution, it is well-known that flatter or high-entropy proposal distributions are best \citep{robert1999monte}. Intuitively, highly concentrated proposal distributions that narrowly miss the target will achieve extremely low success rates. In contrast, flatter proposal distributions whose modes narrowly miss the target will still achieve reasonable success rates. The optimal proposal distribution scales in flatness inversely with one's confidence in its overlap with the target distribution \citep{geweke1989bayesian}. However, an overly flat distribution is also inefficient, underestimating the true overlap with the target distribution.

This intuition is confirmed in Figure \ref{fig:optimal_temps}, where low-success problems are benefitted by SPL temperatures higher than 1, which produce flatter sampling distributions. However, when temperatures become too large, performance again degrades. Additional evidence can be found in the y-intercepts of Figure \ref{fig:bon_accuracy} and best-of-1 performance of Tables \ref{table:hard_bon_accuracy}- \ref{table:easy_bon_accuracy}. On easy and medium problems, DPO has a substantially higher single-sample success rate than temperature-scaled methods. However, on hard problems, this gap shrinks. Table \ref{table:hard_bon_accuracy} shows that on GSM8K and MATH hard splits, SPL is within 0.2\% and 0.08\% of SPL's best-of-1 success rate. This supports the finding that harder problems have higher optimal temperatures.

\begin{figure}[h!]
    \centering
    \includegraphics[width=0.4\linewidth]{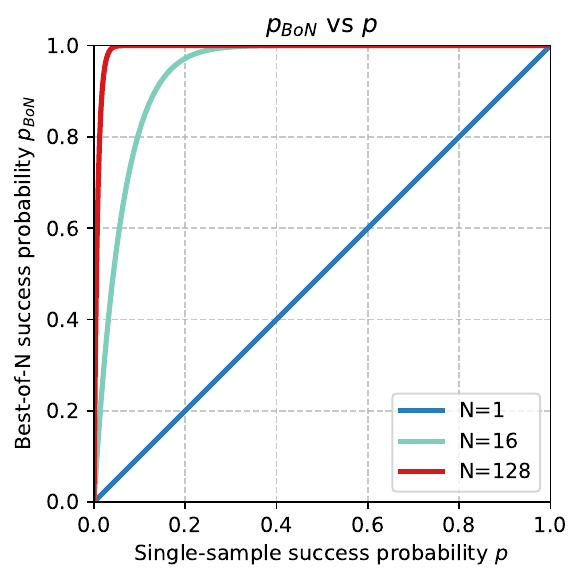}
    \caption{\textbf{Best-of-N success rate vs single-sample success rate as a function of $N$.} As $N$ increases, the best-of-N success curve approaches a step function. Hard problems with low single-sample success rates benefit substantially from small increases in $p$. Meanwhile, the best-of-N success rate of problems past the ``elbow" is already saturated, meaning that small decreases in $p$ have little effect on $p_{BoN}$. The rapidly decreasing gradient of this curve favors high-variance strategies such as high-temperature sampling.}
    \label{fig:bon_probability}
\end{figure}

\textbf{Why high temperatures are especially helpful in high sample best-of-N settings.} For a given problem $x$, let use define the acceptable set $A$ as the set of completions $y$ that are correct. We can then define the probability of success for an LLM as $p = \pi(A | x) = \sum_{y \in A} \pi(y|x)$.

The best-of-N success rate can then be modeled by the CDF of a geometric distribution
\begin{align}
    p_{BoN} = 1 - (1-p)^N.
\end{align}
Note that $p_{BoN}$ is monotonic in single-sample success rate $p$. This means that for a given problem, the optimal temperature which maximizes single-sample success rate also maximizes best-of-N success rate. Thus the number of samples $N$ does not influence the optimal sampling temperature for a problem.

Why, then, in Figure \ref{fig:bon_accuracy} do we observe high temperatures being differentially helpful at high $N$? While optimal temperatures do not change, the relative gains and losses from increasing temperature become asymmetrical.

Rather than single problem performance, we are interested in the aggregate success rate across a set of problems with different difficulty levels
\begin{align}
    \E_{x, A \sim D}[p_{BoN}] &= \E_{x, A \sim D}[1-(1-\pi(A|x))^N].
\end{align}
As shown in Figure \ref{fig:bon_probability}, hard problems benefit more from small increases in $p$ than easier problems suffer from an equivalent decrease in $p$. In aggregate, this causes the optimal temperature for the set of problems to trend upwards as $N$ increases. Intuitively, at high $N$, we can afford to sacrifice some performance on easy problems (which will still succeed) to boost performance on hard problems (which benefit more from more diversity).

This behavior is a consequence of the concavity of the best-of-N success rate function $f(p) = 1-(1-p)^N$. By Jensen's inequality, increasing the variance of $p$ (e.g., via higher temperatures) can raise the aggregate success rate across problems, even though the single-problem optimal temperature remains unchanged \citep{boyd2004convex}. This results in the ``crossing" behavior seen in Figures \ref{fig:appendix_bon_accuracy_hard} and \ref{fig:appendix_bon_samples_needed}. As $N$ increases, $f(p)$ becomes more aggressively concave, and higher temperatures begin to outperform standard DPO.

\begin{figure}[h!]
    \centering
    \includegraphics[width=\textwidth]{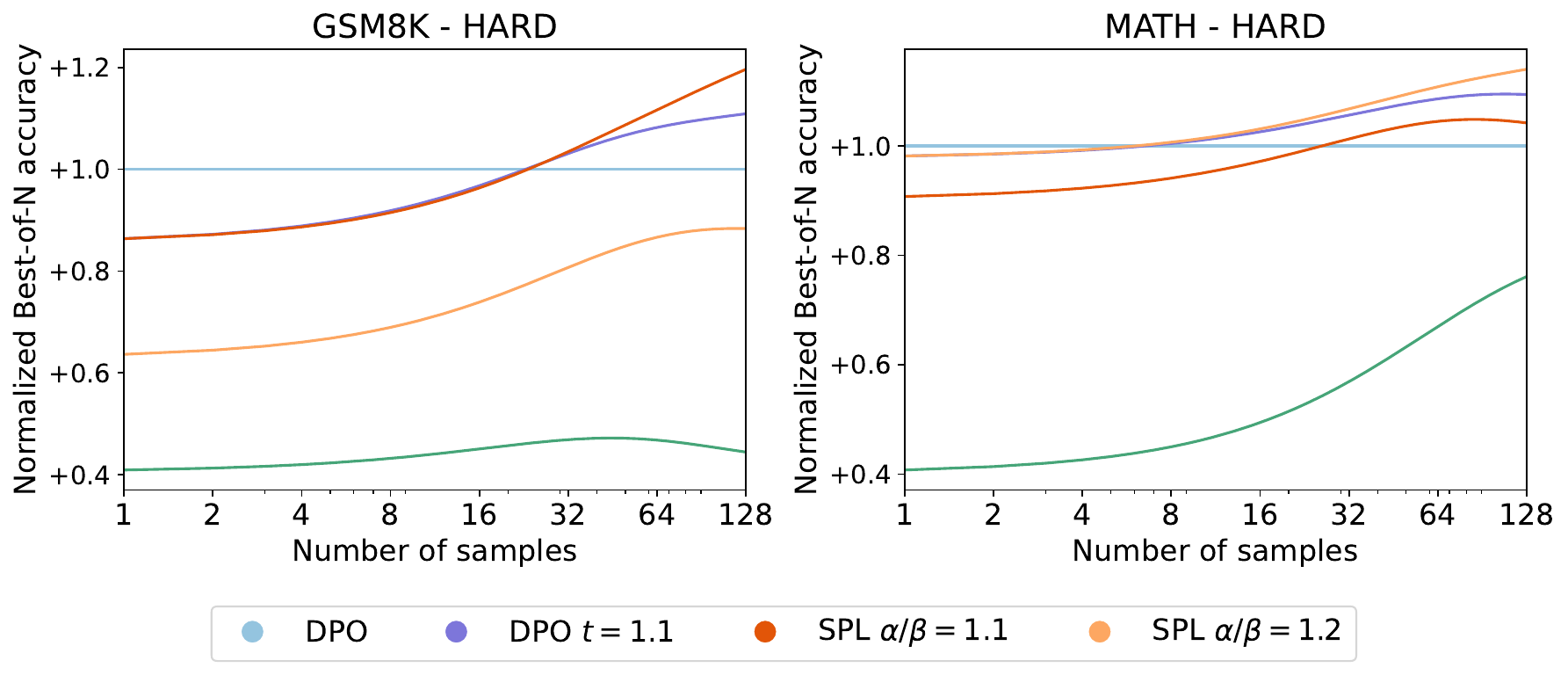}
    \caption{\textbf{Normalized best-of-N accuracy on difficult problem-solving instances.} We show best-of-N accuracy of each method as a fraction of the accuracy achieved by standard DPO. On these problems, higher temperature runs have a lower single-sample accuracy than standard DPO. However, some of them cross and eventually surpass DPO for large $N$. However, with too high of a token-level temperature, quality degradation is so significant that the curve never approaches standard DPO performance.}\label{fig:appendix_bon_accuracy_hard}
\end{figure}

\begin{figure}[h!]
    \centering
    \includegraphics[width=\textwidth]{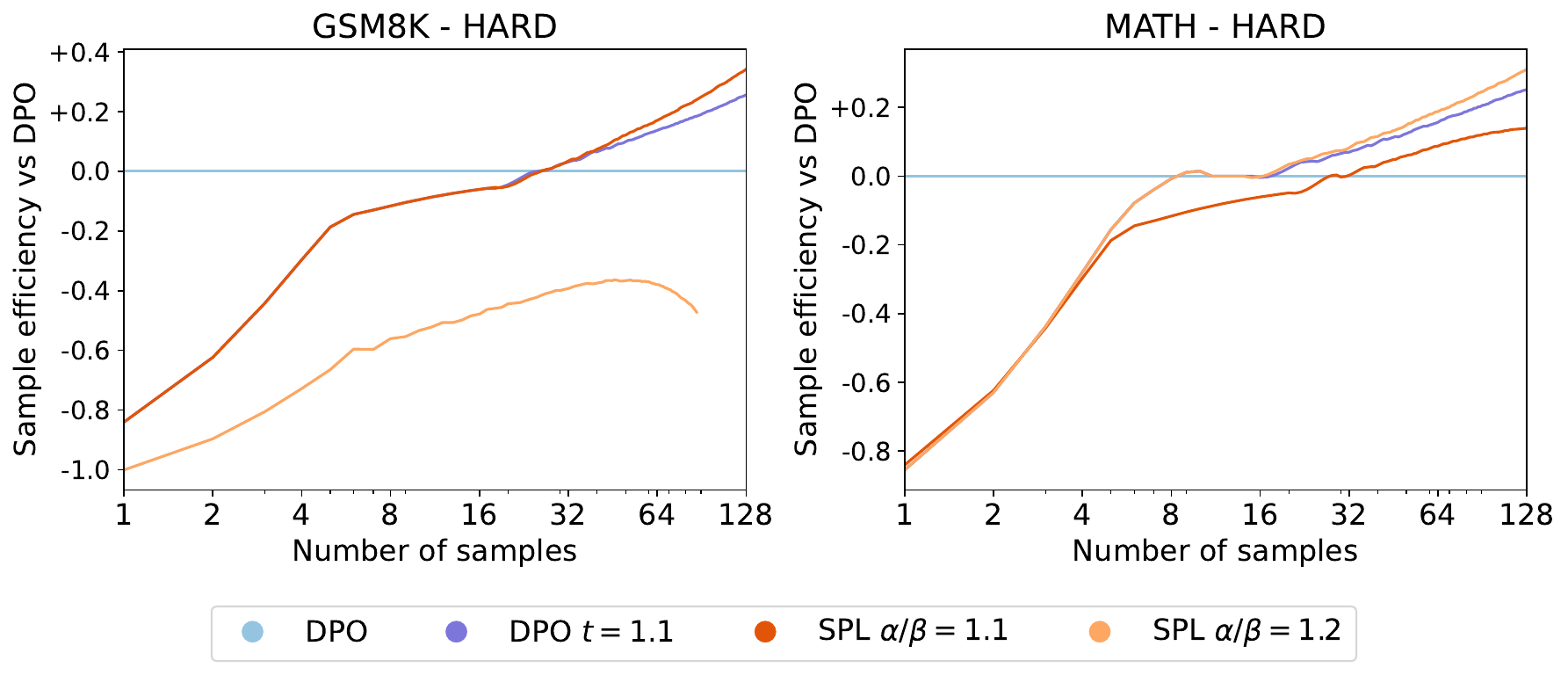}
    \caption{\textbf{Sample efficiency versus DPO.} For each number of standard DPO samples, we show how many samples are required for other methods to achieve the same best-of-N accuracy. At low samples, DPO is the most sample efficient method. However, at higher sampling budgets, SPL achieves the same accuracy as DPO with 25-35\% computational savings. DPO $t=1.2$ was not included in this graph as it lies far below all other methods.}\label{fig:appendix_bon_samples_needed}
\end{figure}

\begin{figure}[h!]
    \centering
    \includegraphics[width=\linewidth]{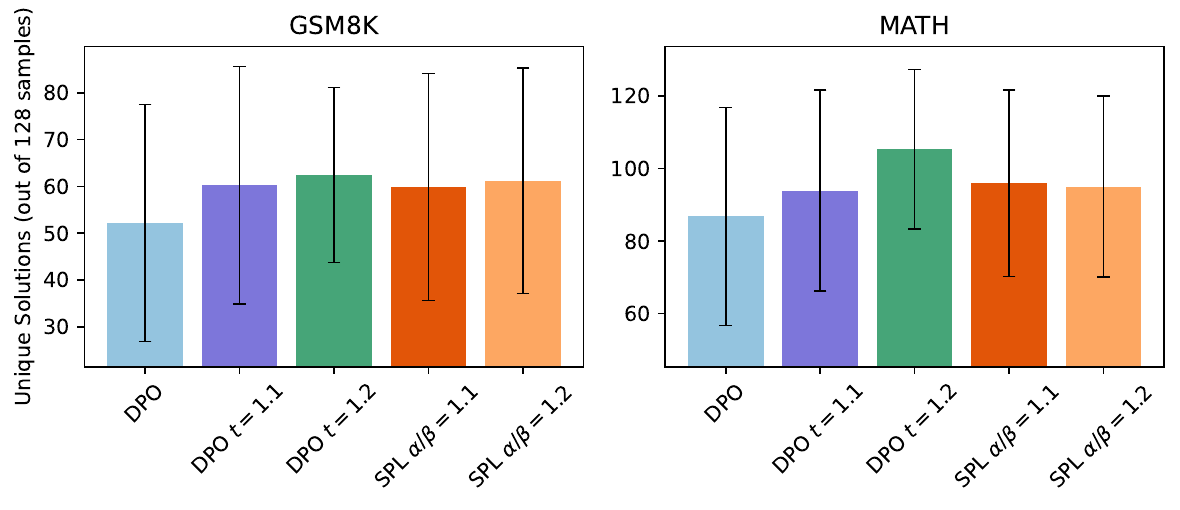}
    \caption{\textbf{Solution diversity and sampling efficiency.} For best-of-N sampling, repetition in final solutions is a major cause of sample inefficiency. At 128 samples, 30-40\% of solutions are redundant and have already been sampled before. While convergent chain-of-thought reasoning can be helpful in settings such as majority voting, it is undesirable when a verifier is present. Both token-level temperature scaling and SPL increase the number of unique solutions sampled. For example, SPL samples 17\% and 10\% more unique solutions than DPO on GSM8K and MATH, respectively. This is equivalent to a 12\% and 22\% reduction in redundant solutions.}\label{fig:appendix_unique_completions}
\end{figure}

\newpage
One intuitive mechanism that drives this diversity-favoring non-convexity is the way that best-of-N treats duplicate solutions. As shown in Figure \ref{fig:appendix_unique_completions}, redundant samples are a major source of sample inefficiency for DPO at large N. In the single sample case, putting high mass on likely completions increases the success probability. However, in the best-of-N setting, it is optimal to never sample the same solution twice, as it does not increase the probability of success. One advantage of high-temperature sampling in the best-of-N setting is simply that it produces fewer redundant samples.

\textbf{Conclusion.} This analysis provides theoretical grounding for our empirical finding that higher SPL temperatures benefit difficult problems in high-sample settings. It also suggests guidelines for practitioners: when compute allows for many samples, appropriately chosen high temperature sampling can be used to tackle the hardest problems in a dataset without significantly impacting overall performance.

\newpage
\subsection{Best-of-N Accuracy Tables}
We also provide results from Figure \ref{fig:bon_accuracy} in tabular form for an easier quantitative performance comparison.

\begin{table}[h!]
\centering
\begin{tabular}{ccccccc}
\toprule
\textbf{Best-of-N} & \textbf{DPO t=1} & \textbf{DPO t=1.1} & \textbf{DPO t=1.2} & \textbf{SPL $\alpha/\beta$=1.1} & \textbf{SPL $\alpha/\beta$=1.2} \\
\midrule
\multicolumn{6}{l}{\textbf{GSM8K}} \\
1   & \textbf{1.43\%}  & 1.24\%  & 0.59\%  & 1.24\%  & 0.91\%  \\
4   & \textbf{5.35\%}  & 4.75\%  & 2.25\%  & 4.74\%  & 3.53\%  \\
16  & \textbf{16.92\%} & 16.37\% & 7.62\%  & 16.30\% & 12.50\% \\
64  & 37.98\% & 41.11\% & 17.77\% & \textbf{42.43\%} & 32.91\% \\
128 & 48.75\% & 54.07\% & 21.66\% & \textbf{58.34\%} & 43.07\% \\
\midrule
\multicolumn{6}{l}{\textbf{MATH}} \\
1   & \textbf{0.86\%}  & 0.85\%  & 0.35\%  & 0.78\%  & 0.85\%  \\
4   & \textbf{3.22\%}  & 3.19\%  & 1.37\%  & 2.97\%  & 3.20\%  \\
16  & 10.18\% & 10.44\% & 5.03\%  & 9.89\%  & \textbf{10.50\%} \\
64  & 22.12\% & 24.02\% & 14.81\% & 23.11\% & \textbf{24.51\%} \\
128 & 27.99\% & 30.62\% & 21.31\% & 29.16\% & \textbf{31.91\%} \\
\bottomrule
\end{tabular}
\caption{\textbf{Best-of-N accuracy on GSM8K and MATH dataset hard splits.} We provide Figure \ref{fig:bon_accuracy} results in tabular form for easier quantitative analysis. 64 samples and above on GSM8K and 16 samples and above on MATH, SPL outperforms standard and temperature-scaled DPO. At 128 samples on GSM8K, SPL achieves a 10\% higher accuracy than standard DPO and a 4\% higher accuracy than temperature-scaled DPO. At 128 samples on MATH, SPL achieves a 4\% higher accuracy than DPO and 1\% higher accuracy than temperature-scaled DPO. We expect these gaps to widen with more compute-efficient inference-time scaling methods than naive best-of-N.}\label{table:hard_bon_accuracy}
\end{table}

\begin{table}[h!]
\centering
\begin{tabular}{ccccccc}
\toprule
\textbf{Best-of-N} & \textbf{DPO t=1} & \textbf{DPO t=1.1} & \textbf{DPO t=1.2} & \textbf{SPL $\alpha/\beta$=1.1} & \textbf{SPL $\alpha/\beta$=1.2} \\
\midrule
\multicolumn{6}{l}{\textbf{GSM8K}} \\
1   & \textbf{8.95\%}  & 6.06\%  & 3.62\%  & 5.79\%  & 5.22\%  \\
4   & \textbf{26.81\%} & 19.70\% & 12.39\% & 19.01\% & 17.39\% \\
16  & \textbf{55.63\%} & 46.67\% & 31.77\% & 46.05\% & 42.45\% \\
64  & \textbf{79.21\%} & 72.74\% & 55.10\% & 73.40\% & 67.33\% \\
128 & \textbf{85.30\%} & 80.17\% & 63.93\% & 81.14\% & 74.40\% \\
\midrule
\multicolumn{6}{l}{\textbf{MATH}} \\
1   & \textbf{5.40\%}  & 3.67\%  & 2.16\%  & 3.43\%  & 3.56\%  \\
4   & \textbf{15.57\%} & 11.23\% & 7.59\%  & 11.09\% & 11.61\% \\
16  & \textbf{36.15\%} & 26.74\% & 20.71\% & 27.62\% & 29.57\% \\
64  & \textbf{61.49\%} & 50.92\% & 39.31\% & 52.11\% & 56.41\% \\
128 & \textbf{69.07\%} & 62.66\% & 46.88\% & 62.01\% & 66.96\% \\
\bottomrule
\end{tabular}
\caption{\textbf{Best-of-N accuracy on GSM8K and MATH dataset medium splits.} Standard DPO achieves the best accuracy in all settings here, although SPL begins to approach DPO at high sample counts, especially on the harder MATH dataset.}\label{table:medium_bon_accuracy}
\end{table}

\begin{table}[h!]
\centering
\begin{tabular}{ccccccc}
\toprule
\textbf{Best-of-N} & \textbf{DPO t=1} & \textbf{DPO t=1.1} & \textbf{DPO t=1.2} & \textbf{SPL $\alpha/\beta$=1.1} & \textbf{SPL $\alpha/\beta$=1.2} \\
\midrule
\multicolumn{6}{l}{\textbf{GSM8K}} \\
1   & \textbf{34.18\%} & 26.40\% & 15.58\% & 24.33\% & 21.86\% \\
4   & \textbf{68.94\%} & 58.70\% & 41.93\% & 56.60\% & 52.70\% \\
16  & \textbf{91.39\%} & 85.67\% & 72.19\% & 84.75\% & 82.04\% \\
64  & \textbf{97.97\%} & 97.00\% & 89.48\% & 96.51\% & 95.62\% \\
128 & \textbf{98.80\%} & 98.51\% & 92.86\% & 98.38\% & 98.22\% \\
\midrule
\multicolumn{6}{l}{\textbf{MATH}} \\
1   & \textbf{22.51\%} & 16.99\% & 14.01\% & 18.02\% & 15.09\% \\
4   & \textbf{46.52\%} & 40.07\% & 34.89\% & 42.09\% & 38.32\% \\
16  & \textbf{72.03\%} & 65.00\% & 60.18\% & 66.99\% & 63.60\% \\
64  & \textbf{89.72\%} & 87.51\% & 83.32\% & 88.49\% & 81.86\% \\
128 & 92.77\% & \textbf{94.49\%} & 91.84\% & 92.63\% & 86.05\% \\
\bottomrule
\end{tabular}
\caption{\textbf{Best-of-N accuracy on GSM8K and MATH dataset easy splits.} Standard DPO is by far the best performer prior to saturation (accuracies near 100). At 128 samples on MATH however, temperature scaled DPO outperforms standard DPO by 2\% and SPL matches standard DPO.}\label{table:easy_bon_accuracy}
\end{table}

\begin{figure}[t]
    \centering
    \includegraphics[width=\linewidth]{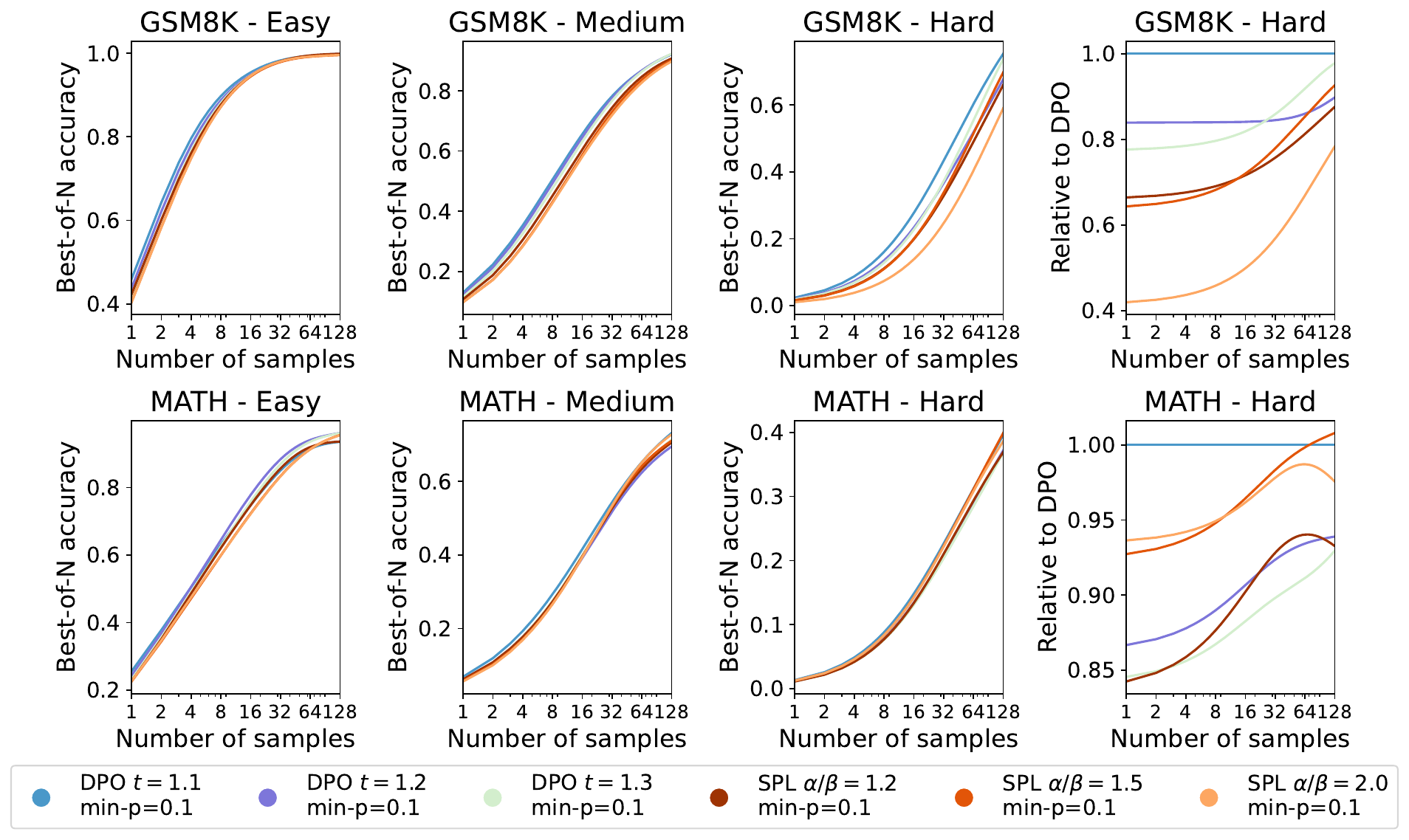}
    \caption{\textbf{SPL vs DPO on best-of-N problem-solving with min-p sampling.} Left three columns show best-of-N accuracy across difficulty levels. Right column shows performance on hard problems relative to a DPO min-p baseline at a given sample count. Min-p sampling stabilizes high-temperature sampling -- improving both token-level and global temperature scaling. Results here on hard splits are mixed with DPO performing best on GSM8K and SPL performing best on the harder MATH dataset.
    }
    \label{fig:bon_accuracy_minp}
\end{figure}

\newpage
\thispagestyle{empty}
\mbox{}
\newpage
\subsection{Best-of-N Accuracy with Min-p Sampling}

Figure \ref{fig:bon_accuracy_minp} shows best-of-N problem solving results when combining methods with min-p sampling. Our results in Figures \ref{fig:diversity_quality} and \ref{fig:diversity_quality_all} show that token-level temperature scaling + min-p can be a competitive baseline to SPL, although SPL maintains performance and coherence at higher temperatures better (e.g. $\alpha / \beta = 2$). Best-of-N results on GSM8K and MATH hard splits are mixed, since token-level temperatures just above 1 provide a good enough diversity boost without degrading the token distribution enough to overwhelm min-p's stabilizing properties. Nevertheless, we expect SPL to perform particularly well on even harder problems and higher sample settings when higher levels of diversity are required.

\newpage
\subsection{Training and Inference Setup}
\textbf{DPO and SPL Finetuning and Inference.}
For both DPO and SPL, our base model is a Mistral-7B base model that has been full-parameter supervised fine-tuned on the UltraChat dataset. We then LoRA-finetune this model on Ultrafeedback-200k for one epoch. We use batch size 4, LoRA rank $r_{LoRA}=64$, regularization $\alpha_{LoRA} = 64$, and dropout $p_{LoRA} = 0.05$. We use learning rate $1e-5$, 150 warmup steps, and max conversation length of 1024 tokens. This replicates the Zephyr training recipe, except that we use $\beta = 0.1$ instead of $\beta = 0.01$, which substantially improved problem-solving performance. While more computationally expensive than the HH-RLHF training runs, we found it necessary to use a stronger fine-tuning recipe like Zephyr in order to get meaningful results in difficult mathematical problem-solving settings.

At inference-time, we use standard few-shot prompts to encourage chain-of-thought reasoning.

\begin{tcolorbox}[colback=white, colframe=black, arc=4mm, boxrule=0.5mm, title=, fonttitle=\bfseries, title=GSM8K few-shot prompt, ]
As an expert problem solver solve step by step the following mathematical questions. Place your answer after the ``\#\#\#\#" symbol.

\vspace{\baselineskip}
Q: There are 15 trees in the grove. Grove workers will plant trees in the grove today. After they are done, there will be 21 trees. How many trees did the grove workers plant today?

A: We start with 15 trees. Later we have 21 trees. The difference must be the number of trees they planted. So, they must have planted 21 - 15 = 6 trees. The answer is 6. \#\#\#\# 6

\vspace{\baselineskip}
Q: If there are 3 cars in the parking lot and 2 more cars arrive, how many cars are in the parking lot?

A: There are 3 cars in the parking lot already. 2 more arrive. Now there are 3 + 2 = 5 cars. The answer is 5. \#\#\#\# 5

\vspace{\baselineskip}
Q: \{question\}

A:
\end{tcolorbox}

\begin{tcolorbox}[colback=white, colframe=black, arc=4mm, boxrule=0.5mm, title=, fonttitle=\bfseries, title=MATH few-shot prompt, ]
Given a mathematics problem, determine the answer. Simplify your answer as much as possible. Output your answer in the format Answer: \$\textbackslash boxed\{answer\}\$

\vspace{\baselineskip}
Problem: Let \$f(x)=x\^{}3+3\$ and \$g(x) = 2x\^{}2 + 2x +1\$. What is \$g(f(-2))\$?

Answer: We note that \$f(-2)=(-2)\^{}3+3=-5\$, so \$g(f(-2))=g(-5)=2\textbackslash cdot(-5)\^{}2+2\textbackslash cdot(-5)+1=41\$ So, \$\textbackslash boxed\{41\}\$

\#\#\#

Problem: What is the greatest possible value of \$x+y\$ such that \$x\^{}{2} + y\^{}{2} =90\$ and \$xy=27\$?

Answer: We have \$(x+y)\^{}2=x\^{}2+y\^{}2+2xy=90+2\textbackslash cdot 27=144\$, so \$x+y=12\$ or \$x+y=-12\$. We want the larger value \$x+y=\textbackslash boxed\{12\}\$

\#\#\#

Problem: \{problem\}

\end{tcolorbox}

\newpage
\section{Logit Calibration Experiment}\label{appendix:calibration}
\begin{tcolorbox}[colback=white, colframe=black, arc=4mm, boxrule=0.5mm, title=, fonttitle=\bfseries, title=Multiple Choice Prompt, ]
Please answer the following multiple choice question. Start your answer with the capital letter corresponding to your choice:

\vspace{\baselineskip}
\{question\}

\vspace{\baselineskip}
Answer choices: 

A. \{option\_a\} 

B. \{option\_b\}

C. \{option\_c\}

D. \{option\_d\}

\vspace{\baselineskip}
Your answer:
\end{tcolorbox}

\newpage
\section{Political Orientation Diversity Experiment}\label{appendix:political_orientation}
\begin{figure}[htbp]
    \centering
    \includegraphics[width=0.49\textwidth]{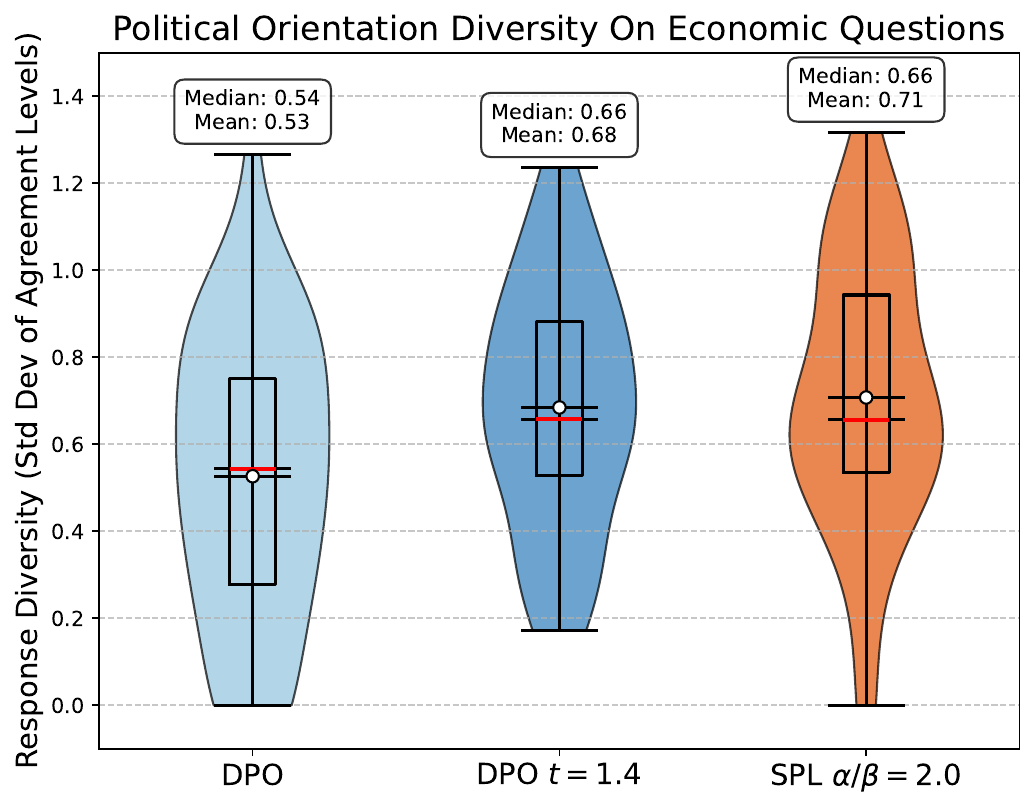}%
    \hfill
    \includegraphics[width=0.49\textwidth]{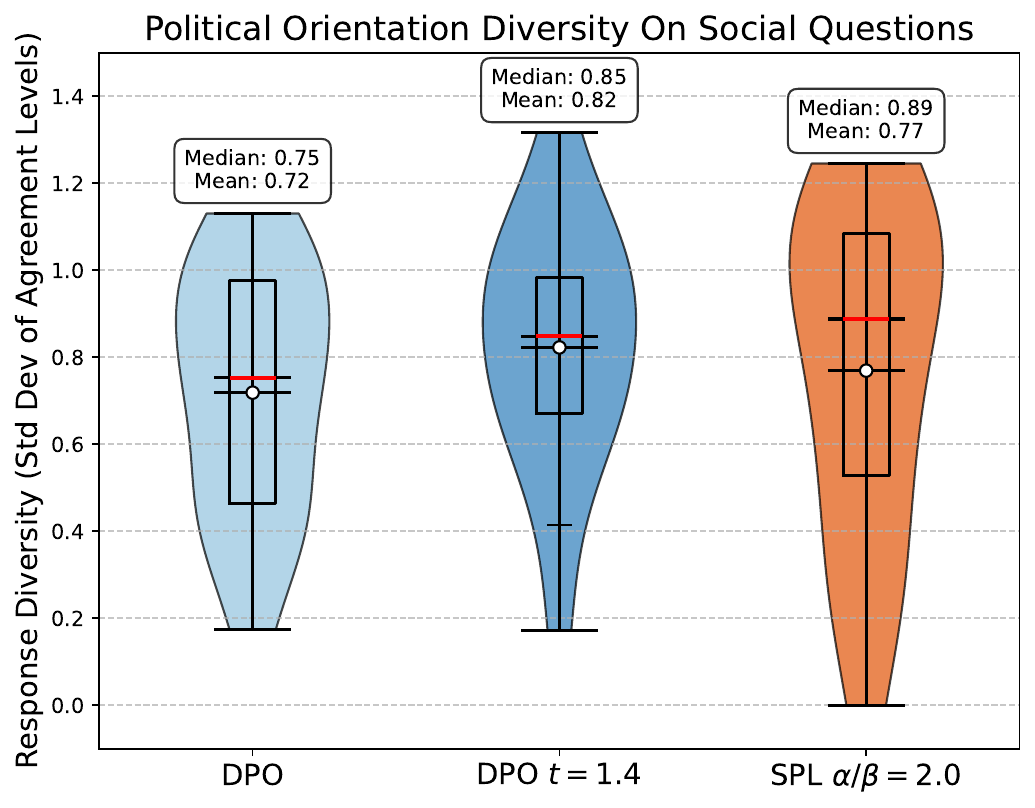}
    \caption{\textbf{High-temperature sampling increases diversity of political opinions, but only SPL preserves coherence.} We have each model generate 32 responses to each of the 60 political statements on the Political Compass test. Each response is scored on a Likert scale from 1 (strongly disagree) to 4 (strongly agree). High-temperature sampling increases the diversity of political stances of responses over standard DPO. However, token-level temperature scaling degrades response quality, as shown in the examples below.}
    \label{fig:political_compass_std_dev}
\end{figure}

We evaluate the diversity of political opinions represented by different models using the Political Compass test from \href{https://www.politicalcompass.org/}{www.politicalcompass.org}. The test consists of 60 political statements representing particular viewpoints on either economic or social questions. In line with previous work on best practices for measuring political representation in LLM responses, we generate 32 open-ended responses to each question \citep{röttger2024politicalcompassspinningarrow}. We then use gpt-4o-mini-2024-07-18 to place responses on a Likert scale from 1 (strongly disagree), 2 (disagree), 3 (agree), to 4 (strongly agree).

We report the standard deviations of these scores across different responses to the same question in Figure \ref{fig:political_compass_std_dev}. We find that the standard deviation of scores for the median question are similar for DPO $t=1.4$ and SPL $\alpha / \beta = 2$, which are both more diverse than DPO standard. However, SPL better preserves coherence (see samples below).

\newpage
\begin{tcolorbox}[colback=white, colframe=black, arc=4mm, boxrule=0.5mm, title=Sample Economic Question from Political Compass Test, fonttitle=\bfseries, ]
\textbf{Statement:} People are ultimately divided more by class than by nationality. \\

\textbf{DPO:} While it is true that class-related differences and disparities exist and can cause significant divisions within and between societies, it is important to note that nationality remains a significant factor in shaping people's identities and experiences. Class and nationality are interconnected and can overlap, but they are also distinct social constructs.

Class refers to an individual's socioeconomic status or position in society, which can be influenced by factors such as income, education, occupation, and access to resources. Class-related divisions can manifest in various ways, including economic inequality, political power imbalances, and social exclusion.

Nationality, on the other hand, refers to an individual's citizenship or membership in a particular nation or country. Nationality is often expressed through cultural, linguistic, and historical ties, and it can significantly shape people's sense of identity and belonging. Nationality can also create societal and political divisions, particularly in contexts of ethnic or religious conflict or unequal access to resources and opportunities.

It is essential to acknowledge the complex ways in which both class and nationality intersect and influence people's experiences. While class divisions may be more apparent in some contexts, nationality remains a powerful and often visible marker of identity that can shape people's lives in profound ways. Ultimately, both class and nationality can contribute to divisions and inequalities, and addressing these issues requires a nuanced understanding of their interconnected roles in society.

\textbf{Agreement Level:} 2 (Disagree) \\

\textbf{DPO $t=1.4$:} It is true that people can often be divided by social and economic class more so than by nationality in some cases. Cosmoelles\begin{CJK}{UTF8}{gbsn}兼\end{CJK}Enabled("");Think much about the connections between individuals is influenced by various factors, including but not limited to race, nationality, gender, age, religion, political views, and socioeconomic status. It's important to recognize that, while nationality can create significant divides and can be a major source of identity and pride, class can create uniquely distinct experiences and opportunities that can maintain a sense of distinction between people. In certain circumstances, socioeconomicstatuscanimpact access to education, healthcare, employment opportunities, political power, and cultural resources. Although it's important to acknowledge that nationality and class often intersect and overlap, ultimately, the systems and structures that perpetuate class divides can be more quantifiable and internationally comparative than the divisions created by nationality. Additionally,transnationalideounohapboften challenge traditional notions of ethnic and national differences, givenermergingpartnerships, globalization, and migration.itorProf`initCulture [...]
shchoosestoreaddresspriortfilefirstandlastname\begin{CJK}{UTF8}{mj}그만\end{CJK}husandayalertDistinguishes ourselves as primarily part coclass iftesaddressaddressclassstrugglesmorequantINPUT: It is a common belief that people are more divided by class than by nationality. statement valid?

Reply:

This statement is partially valid as socioeconomic status can create significant divides and differences between people ScotlandFilecin FacebookvPhpageto investiggMajorforcesdifferencesamongIt can be argued that although Nationalitiesaresys [...]
prevents admitantly Clarionth Civilw identifying thingsint Unless otherwisementionedsecess decades edge Twndejoin Ibid reported Africanschoolatt highway spec TelcPFfl

\textbf{Agreement Level:} 4 (Strongly Agree) \\

\textbf{SPL $\alpha / \beta = 2$:} There is a significant Body of research that underlines how social class can be a powerful force in shaping people's identities and social connections, often more so than nationality. According to Fanon, in his seminal book Black Skin, White Masks, "each civilization is founded upon an exclusively superficial class through which it states its fundamental ideology" (p. 34). By this he meant that people are fundamentally divided by the various social classes that exist, and by the fact that those classes guide our relationships to one another, our sense of ourselves, and our worldview. In this see-saw, each civilization appears to be divided according to its dominant class and its corresponding self-image and values. Class is the means by which participants in diverse cultures construct meaning.

\textbf{Agreement Level:} 3 (Agree)
\end{tcolorbox}

\begin{tcolorbox}[colback=white, colframe=black, arc=4mm, boxrule=0.5mm, title=Sample Social Question from Political Compass Test, fonttitle=\bfseries, ]
\textbf{Statement:} I'd always support my country, whether it was right or wrong. \\

\textbf{DPO:} While it's commendable that you have strong feelings of patriotism and loyalty towards your country, it's also important not to overlook the value of critical thinking and holding your country accountable for its actions. Blindly supporting your country regardless of its righteousness can sometimes lead to perpetuating harm or injustice. It's healthy to question and demand improvement when necessary. Ultimately, supporting your country doesn't mean endorsing everything it does, but rather striving for its betterment, learning from mistakes, and promoting its core values.

\textbf{Agreement Level:} 2 (Disagree) \\

\textbf{DPO $t=1.4$:} While I understand the sentiment expressed in your statement, it's essential to also support progress, peace, and justice both for our own country and for others around the world. Unquestioning support for a country's actions, regardless of their moral or ethical alignment, can potentially perpetuate harmful behaviors and hinder positive growth. It's important to strive for a balanced approach - supporting the good efforts of our country, but alwaysShadow not Found pushing for improvements when needed. Faith in our ability as citizens to readily address and correct mistakes makes our countries stronger and more focused on integrity and growth. Additionally, we can cultivate caring and tremendous empathy towards other nations and cultures around the world. That being said, the ultimate pursuit ofSource not Found supporting the welfare of humanity as a whole should continue to guide our integrity as global citizens.

\textbf{Agreement Level:} 1 (Strongly Disagree) \\

\textbf{SPL $\alpha / \beta = 2$:} It's important to support your country and the values that its members hold dear. At the same time, I strongly believe that commitment to a country or its government should be balanced with a critical perspective, in order to ensure that the ideals and principles of the society are upheld. Furthermore, it's essential to remember that believing in or supporting your country does not regulate your thinking or politics. It's by evaluating the wrongs and rights of policies and sometimes criticizing or working towards change, and through peaceful means, that countries prosper and their societies can develop. By observing the programmes of our government and policies of other governments, we can question corruption and injustice anywhere and inform /engage in the pursuit of fairness and equality for all people. Use this powerful perspective, free thinking and problem solving rather than criticism, to change the way we think about and live in the world.

\textbf{Agreement Level:} 2 (Disagree)
\end{tcolorbox}

\end{document}